\documentclass{alt2021}
\usepackage{forest}
\usepackage{framed}
\usepackage{wrapfig}
\usepackage{graphicx}
\usepackage{xfrac}
\usepackage{tikz,pgfplots}
\pgfplotsset{compat=newest}
\usepackage[utf8]{inputenc}
\usetikzlibrary{patterns}
\usepackage{amsmath}
\usepackage{amssymb}
\usepackage{colortbl}
\usepackage{cancel}
\usepackage{ulem}
\usepackage{multirow}
\usepackage{relsize}
\usepackage{algorithm}
\usepackage{algorithmic}
\usepackage{forloop}
\usepackage{bm}
\usepackage{mathtools}
\newcounter{loopcntr}
\usepackage{booktabs}
\newcommand{\ELL}{q}

\newcommand{\Xtr}{\X_{\text{tr}}}

\newcommand{\ytr}{\y_{\text{tr}}}

\newcommand{\PPi}{\boldsymbol{\P}}
\newcommand{\Qed}{\hspace*{\fill}$\blacksquare$\par\smallskip}

\newcommand{\yh}{\hat{y}}

\renewcommand{\v}{\bm{v}}
\renewcommand{\a}{\bm{a}}

\newcommand{\M}{\bm{M}}
\newcommand{\h}{\bm{h}}

\usepackage{mathtools}

\newcommand{\zero}{\bm{0}}

\DeclareMathOperator{\sign}{sign}
\DeclareMathOperator{\tr}{tr}

\newcommand{\del}{\bm{\delta}}
\newcommand{\y}{\bm{y}}

\DeclareMathOperator{\diag}{diag}

\newcommand{\rank}{\text{\normalfont rank}}

\newcommand\sbullet[1][.5]{\mathbin{\vcenter{\hbox{\scalebox{#1}{$\bullet$}}}}}

\newcommand{\wstardot}{%
  \mathrel{\vbox{\offinterlineskip\ialign{%
    \hfil##\hfil\cr
    $\scriptscriptstyle\sbullet[0.4]$\cr
    \noalign{\kern+.4ex}
    $\w$\cr
}}}^*\!\!(t)}

\DeclareMathOperator*{\argmin}{\mathop{\mathrm{argmin}}}

\newcommand{\half}{\sfrac12}

\renewcommand{\u}{\bm{u}}

\newcommand{\s}{{\bm{s}}}
\newcommand{\st}{{\tilde{\s}}}
\newcommand{\x}{\bm{x}}
\newcommand{\e}{\bm{e}}
\renewcommand{\H}{\bm{H}}
\newcommand{\Ht}{\bm{\widetilde{H}}}
\newcommand{\Z}{\bm{Z}}

\newcommand{\X}{\bm{X}}

\newcommand{\Y}{\bm{Y}}

\newcommand{\I}{\bm{I}}

\renewcommand{\v}{\bm{v}}

\renewcommand{\t}{{\bm{\theta}}}

\newcommand{\one}{\bm{1}}

\newcommand{\A}{\bm{A}}
\renewcommand{\b}{\bm{b}}

\newcommand{\RR}{\mathbb{R}}

\newcommand{\U}{\bm{U}}
\renewcommand{\P}{\bm{P}}
\newcommand{\W}{\bm{W}}

\newcommand{\w}{{\bm{w}}
}

\definecolor{darkgreen}{rgb}{0.09, 0.45, 0.27}

\DeclareMathOperator{\Mb}{\mathbf{M}}
\DeclareMathOperator{\EE}{\mathbb{E}} 
\DeclareMathOperator{\MM}{\mathbb{M}} 

\begin{document}
\title[Spindly linear network whips any neural network]{A case where a spindly two-layer linear network\\
whips any neural network with a fully connected input layer}
 \altauthor{%
  \Name{Manfred K. Warmuth} \Email{manfred@google.com}\\
 \addr Google Research, Mountain View, CA
 \AND
 \Name{Wojciech Kot{\l}owski} \Email{kotlow@gmail.com}\\
 \addr Poznan University of Technology, Poznan, Poland
 \AND
 \Name{Ehsan Amid} \Email{eamid@google.com}\\
 \addr Google Research, Mountain View, CA}


\maketitle
\begin{abstract}
    It was conjectured that any neural network
    of any structure and arbitrary differentiable
    transfer functions at the nodes cannot learn the following
    problem sample efficiently when trained with gradient descent:
    The instances are the rows of a $d$-dimensional Hadamard matrix
    and the target is one of the features, i.e. very sparse. We essentially prove this
    conjecture:
    We show that after receiving a random training set of size $k < d$, the expected square loss is still $1-\sfrac k{(d-1)}$. The only
    requirement needed is that the input layer is fully
    connected and the initial weight vectors of the input nodes
    are chosen from a rotation invariant distribution.

    Surprisingly the same type of problem can be solved
    drastically more efficient by a simple 2-layer
    linear neural network in which the $d$ inputs are connected to
    the output node by chains of length 2 (Now the input
    layer has only one edge per input). When such a network
    is trained by gradient descent, then
    it has been shown that its expected square loss is $\frac{\log d}{k}$.

    Our lower bounds essentially show that
    a sparse input layer is needed to sample efficiently learn sparse targets
    with gradient descent when the number of examples is
    less than the number of input features.
\end{abstract}
\pagestyle{plain}
\maketitle



\maketitle

\section{Introduction}
Neural networks typically include fully connected layers as part of their architecture.
We show that this comes at a price. Networks with a fully
connected input layer are {\it rotation invariant} in the
sense that rotating the input vectors does not affect the
gradient descent training and the outputs produced by the
network during training and inference.
More precisely for rotation invariance to hold, the distributions used for choosing the
initial weight vectors for the fully connected input layer must be zero or
rotation invariant as well.
With this mild assumption, we will show that if the input vectors are rows of a $d$-dimensional Hadamard matrix and if the target is one of the $d$ features
of the input, then after seeing $k$ examples
any such gradient descent trained network has expected square loss at least
$1-\sfrac k{(d-1)}$. That is, after seeing $\sfrac {(d-1)}2$ examples,
the square loss is at least $\sfrac12$.
Such a hardness result was conjectured in~\citep{span2} for
any neural network trained with gradient descent without the
additional assumption that the input layer is fully connected and initialized
by a rotation invariant distribution.

The lower bounds proven here are complemented by a recent result in~\citep{wincolt20} which shows that a simple sparse 2-layer linear neural
network when trained with gradient descent
on $k$ examples has expected square loss $\frac{\log d}{k}$.%
\footnote{For the sake of simplicity we only quote the
bound when the target is a single noise-free feature. The gradient
algorithm does one pass over the examples and after each of the
$k$ examples is processed, forms a hypothesis by clipping its
predictions. The bound is proven using an on-line-to-batch
conversion, i.e. the algorithm predicts randomly with one
of the past $k$ hypotheses.}
The network in question can be seen as a single neuron
where each of the $d$ edges from the inputs
\begin{forest}
    for tree={circle,draw,scale=.40,grow=0}
    [[]]
\end{forest}
is replaced by a duplicated (or ``squared'') edge
\begin{forest}
    for tree={circle,draw,scale=.40,grow=0}
    [[[]]]
\end{forest}
(Figure \ref{fig:spindly}).

\begin{wrapfigure}{r}{0.36\textwidth}
        \vspace{-0.7cm}
\begin{framed}
    \begin{center}\hspace{-0.2cm}output $\yh$
    \end{center}
        \vspace{-0.7cm}
    \begin{center}
\begin{forest}
for tree={circle,draw,l sep=18pt,s sep=27pt, scale=.45}
[,
    [ [ ] ]
    [ [ ] ]
    [  ,edge label={node[midway,left=-1mm] {$u_i$}}
       [ ,edge label={node[midway,left=-1mm] {$u_i$}}
         ] ]
    [  [ ] ]
    [  [ ] ]
]
\end{forest}
       \end{center}
\vspace{-.4cm} \hspace{1.8cm} input $\x$
   \vspace{-.1cm}
\caption{Reparameterizing the weights $w_i$ of a
            linear \mbox{neuron} by
	    $u_i^2$.}\label{fig:spindly}
\vspace{-0.2cm}
        \end{framed}
        \vspace{-0.78cm}
\end{wrapfigure}
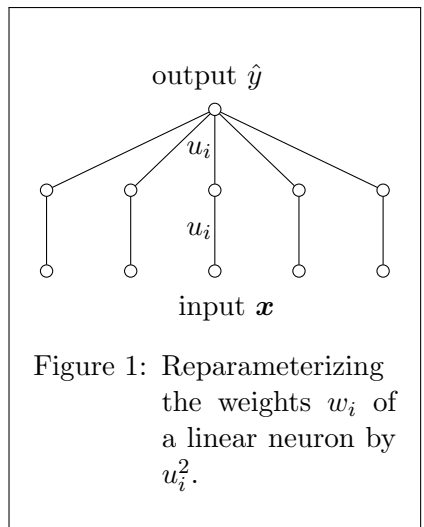

We introduce a number of novel lower bounding techniques in
this paper. The lower bounds have two components:
First, a mechanism is needed for assuring that the neural network does not
get some key information about which examples it has received.
This is achieved by the orthogonality of the instances and
by our assumptions about the neural networks that makes the
predictions of the network invariant to rotating the
inputs, i.e. fully connected input layer and with rotation invariant initialization
that is trained with gradient descent.
Second, we must assure that the network cannot infer the target from
the labels (while the number of training examples $k$ is less than $d$).
We do this by randomly sign flipping or
complementing the $k$ instances that are given to the network.
An alternate method is to permute the instances.
With these two components in place, we show that such neural
networks still have loss $1-\sfrac {k}{(d-1)}$ after seeing $k < d$
examples.
In the lower bounds, our assumptions assure that each input weight vector equals its
initialization plus a linear combination of the past instances.
Since the unseen instances are orthogonal and the network is
uncertain about the target, the best way to predict on the
unseen instances is zero and this gives the lower bound.
When the instances are rows of Hadamard matrices or
shifted bit versions of these matrices then
lower bounds hold for learning a single feature.
We also show that a slightly weaker $(1-\sfrac kd)^2$ lower bound holds
for learning a single feature when the components of the instances are Gaussian
  instead of $\pm1$-valued orthogonal instances as for Hadamard.

In a second part we greatly expand the SVD based lower
bound technique of~\cite{span}. This technique
is for learning features of the Hadamard instances,
however now rows of the Hadamard matrix can be expanded
with an arbitrary $\phi$ map and we do not require rotationally invariant initialization.
When learning Hadamard features with fully connected
linear neural networks using gradient descent,
then the rank $r$ of a certain matrix is
the key limiting factor: The $i$-th column of that
matrix is the combined weight vector after seeing $k$ examples
with target feature $i$.
One can show that the average loss of the gradient descent
based algorithm over the $d$ target features is
at least $1-\sfrac r d$ and the
rank for a single layer is at most $k+1$ and for two
layers, at most $2k+1$ (These rank bounds hold for any
initialization).
\begin{wrapfigure}{r}{0.49\textwidth}
        \vspace{-0.5cm}
\begin{center}
\includegraphics[width=.2\textwidth]{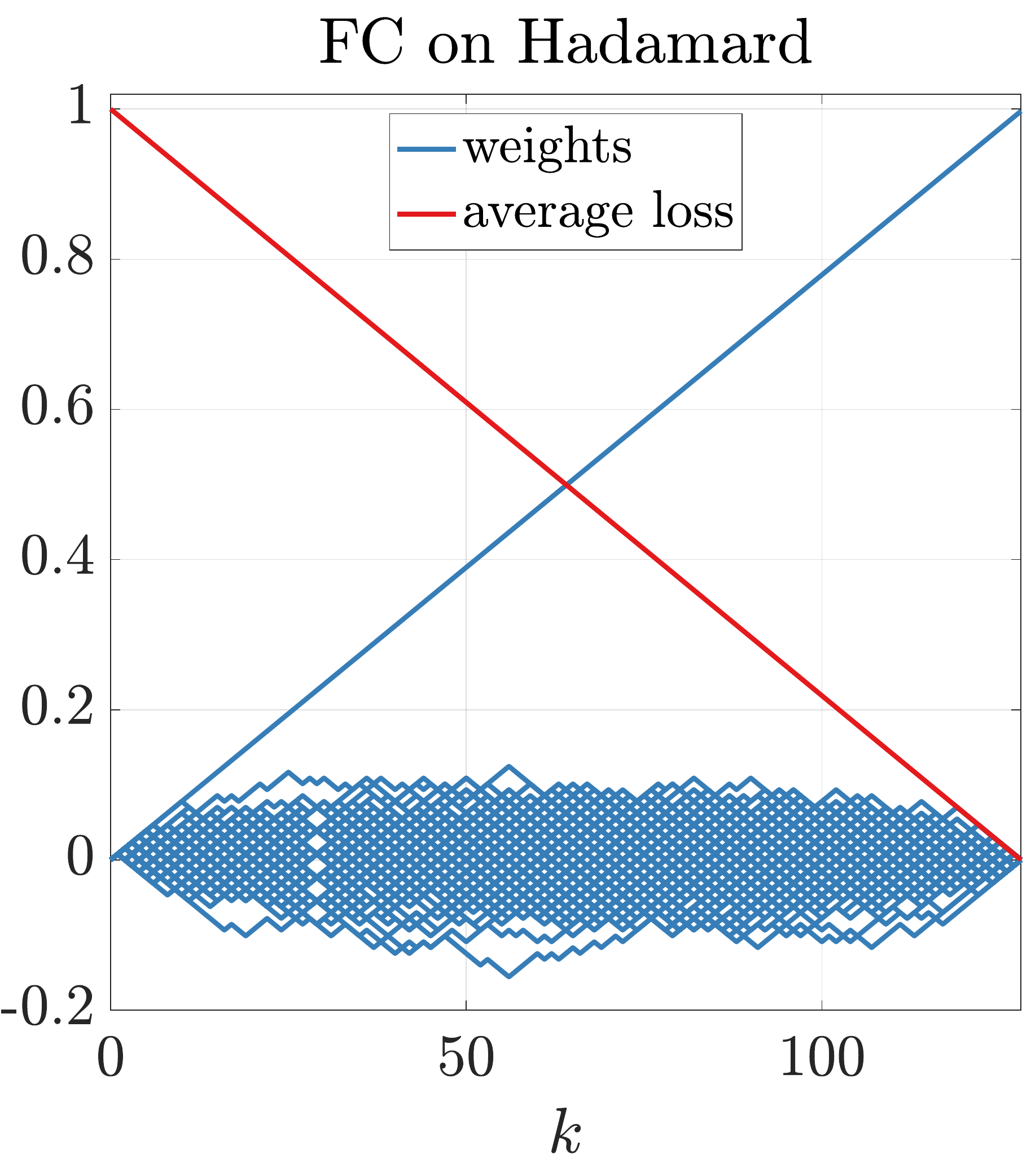} \includegraphics[width=.2\textwidth]{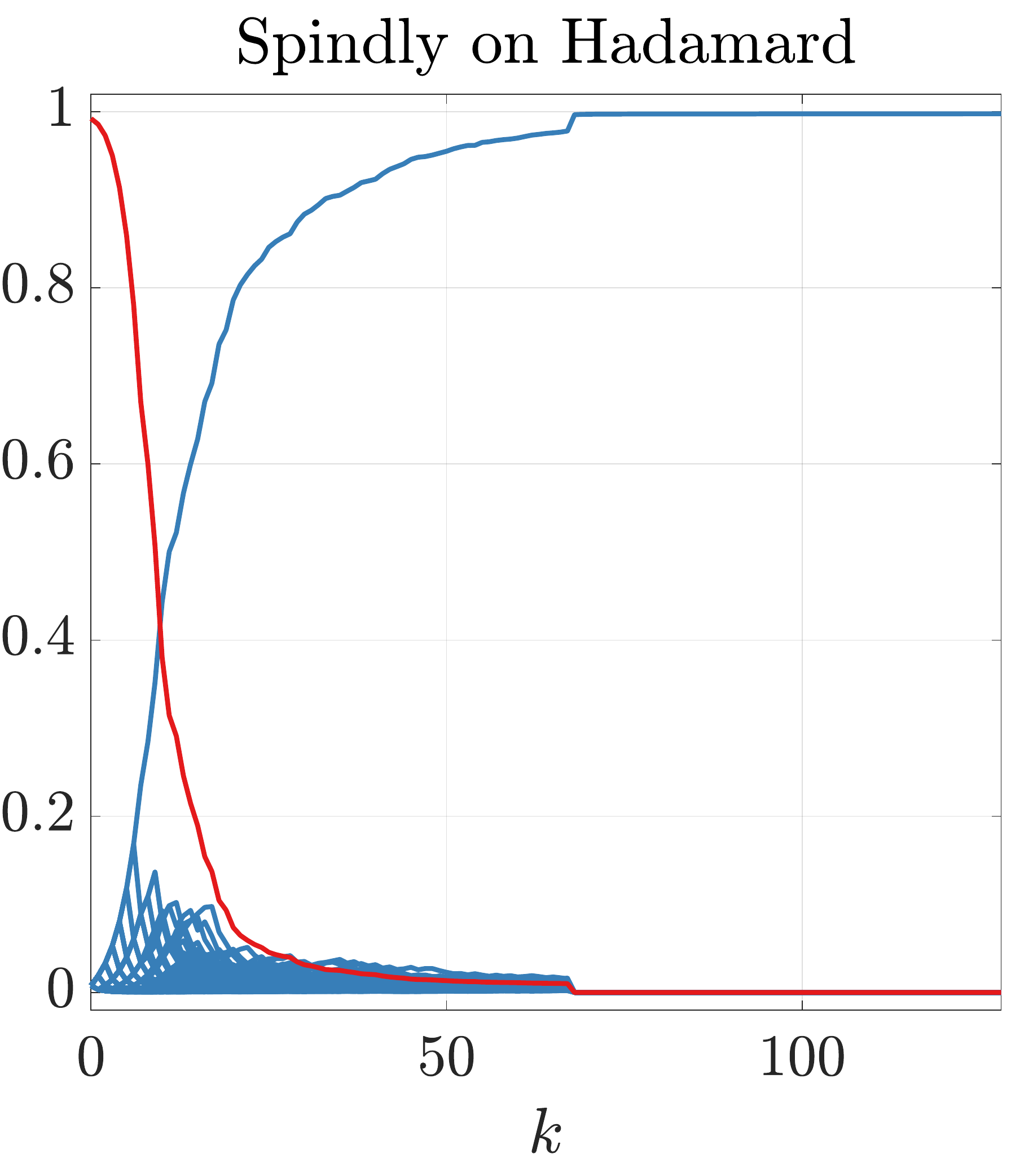}\\
\vspace{-0.3cm}
\includegraphics[width=.2\textwidth]{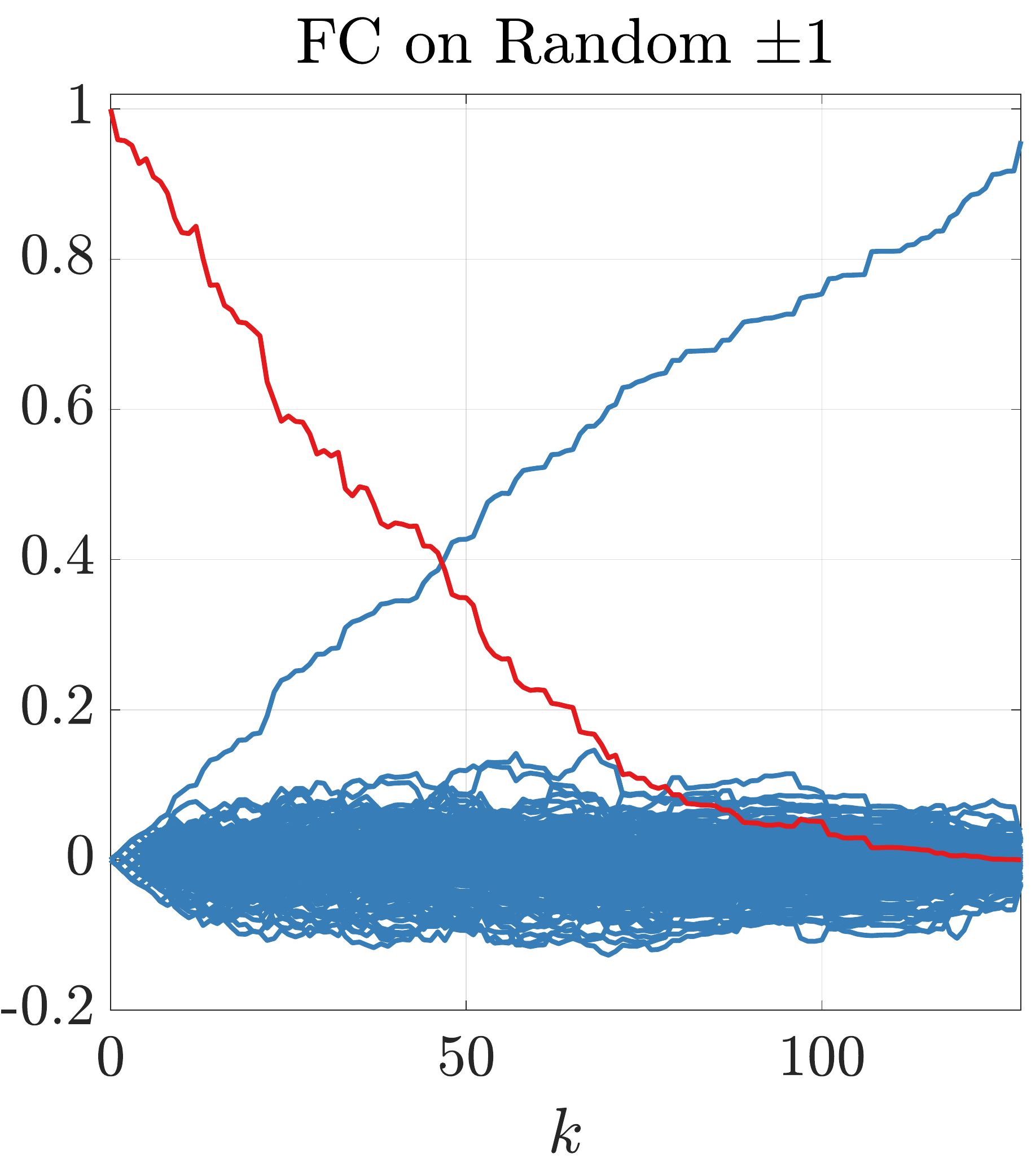} \includegraphics[width=.2\textwidth]{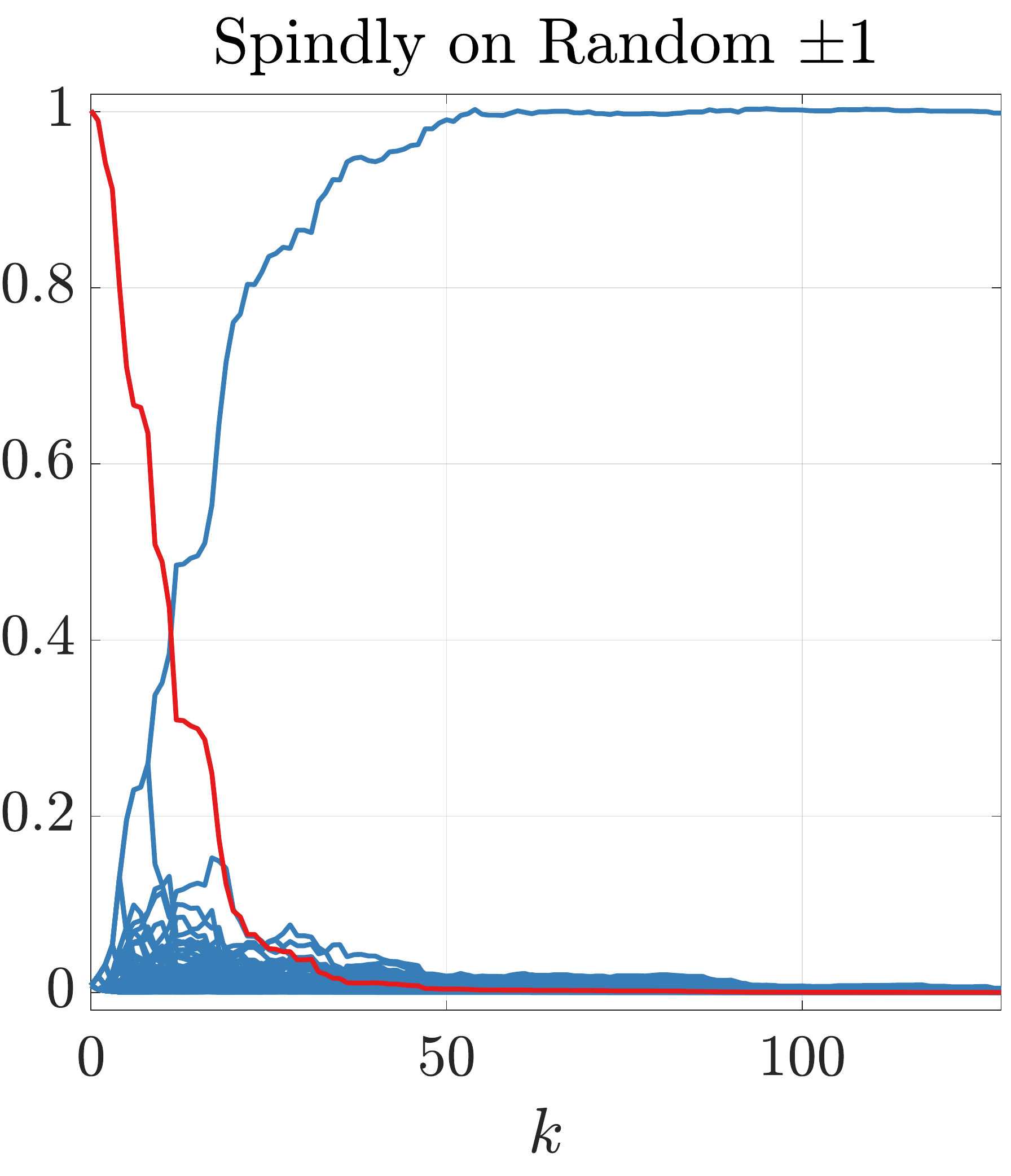}
\end{center}
\vspace{-0.7cm}
\caption{SGD on the Linear fully connected neuron (left) and spindly network
(right) on Hadamard (top) and random $\pm 1$ matrices (bottom).}\label{fig:gd_vs_spindly}
\vspace{-0.5cm}
\end{wrapfigure}
This second technique gives incomparable lower bounds. In
particular it gives linear lower bounds when the
instance matrix is random $\pm 1$ or $\{0,\!1\}$ or Gaussian instead of Hadamard.
However the rank technique only applies to
gradient descent trained fully connected linear neural networks of up to two layers.
The rank $r$ is not easily bounded already for three layers.
In contrast, the rank immediately becomes $d$ after seeing $\log d$ many examples
when learning with the spindly network.

    Curiously, the Hadamard problem can be seen as an exponential expansion of a
    cryptographically secure $\log d$ bit XOR problem~\citep{xor}.
    This allows multiplicative updates
    and their gradient descent reparameterizations as the
    spindly network to learn this problem sample efficiently,
    while no expansion allows kernel method to avoid the
    linear lower bounds. We expand on this in Appendix
    \ref{a:hadam}.

\paragraph{Previous work}

The key idea of using squared weight parameters goes back to
\citet{srebro1}, who showed convergence of continuous gradient descent (CGD)
on the squared weight reformulation to the
minimum $\mathrm{L}_1$-norm solution in a matrix factorization context.
The matrix analysis includes the spindly network as a special case
where the factors are diagonal. It was shown later in~\citep{regretcont}
that CGD on the spindly network equals to continuous
unnormalized exponentiated gradient update (EGU) on a single neuron.%
\footnote{EGU is the paradigmatic multiplicative update
(a mirror descent update based on the $\log$ link).}
\citet{wincolt20} also showed worst case regret bounds for the spindly network
with discrete GD updates. These bounds closely match the original bound shown previously
in~\citep{eg} for EGU.
In contrast, linear lower bounds for GD on a single linear neuron were
shown in \citep{span} when the instances are the $\pm 1$ valued rows of a Hadamard matrix
and the target is a single feature.
The linear behavior of SGD for single linear neurons holds
experimentally even when the instances are random
$[-1,1]^d$ or $[0,1]^d$ vectors. On the other hand, the
regret bounds for the spindly network was shown for the
case where instances are from the domain
$[0,1]^d$.\footnote{More generally, for $[0,R]^d$ for some
$R > 0$.} Nevertheless, the exponential decay of the average
loss of SGD on the spindly network can be shown
experimentally even when instances are random $[-1,1]^d$ vectors (see Figure~\ref{fig:gd_vs_spindly} for a comparison). Note that~\citet{wincolt20} discussed
this observation only on an experimental level. Also, the
domain mismatch in the upper and lower bounds was not addressed
in the previous work. In order to close this gap, we extend the linear lower bound
for a single linear neuron trained with GD to the case where the instances are in $[0,1]^d$.
The technique for proving this type of lower bound
(called second technique above) was originally introduced in \citep{span}.
By averaging over targets, the linear lower bound was shown for learning Hadamard
features even if the instances are embedded by an arbitrary $\phi$ map,
thus showing that Hadamard features cannot be learned by any kernel method.
This type of lower bound relies on the flat SVD spectrum of the Hadamard matrix.
The spectrum of random $\pm 1$ matrices is also
sufficiently flat to lead to a lower bound for any kernel
methods.\footnote{%
With some additional combinatorial techniques, the method
also leads to linear lower bounds for single neurons with
essentially any transfer function \citep{span2}.}
In this paper, we apply this second proof methodology to obtain similar
linear lower bounds for GD trained 2-layer fully connected linear neural networks
with arbitrary initialization.

Note that this paper is not per se about hardness results
for learning certain unusual functions with neural networks \citep{matus,shamir}.
Instead we build on the work of \citet{span,rotinv} that
study algorithms i.t.o. their invariance properties
and prove linear lower bounds for rotation invariant algorithm
(that include GD trained neural networks with fully connected input
layers). These lower bounds can be bypassed by sparsifying
the input layer (Figure~\ref{fig:spindly}).

\paragraph{Outline.} We begin with some basic definitions in Section~\ref{s:notation}
and proceed to prove our single feature lower bounds in Section~\ref{s:single}.
We then discuss the weaknesses of these bounds and then give extensions of the SVD based
lower bound techniques (mostly relegated to the appendices). Open problems are discussed in the conclusion section.
\section{Notations and setup}
\label{s:notation}
    An {\it example} $(\x,y)$ consists of a $d$-dimensional
    input vector $\x$ and real label $y$.
    We specify a learning problem as a tuple
    $(\underset{n, d}{\X},\underset{n}{\y})$ containing $d$ examples,
    where the rows of \emph{input matrix}
    $\X$ are the $n$ (transposed) input vectors and the \emph{target} $\y$ is a vector
    of their labels. For the sake of proving lower bounds, we will also
    consider learning problems having multiple targets
    $(\underset{n, d}{\X},\underset{n, m}{\Y})$,
    where now the columns of $\Y$ specify the separate targets.
    The training set always consists of the first $k\le n$
    examples of the problem $(\underset{k,d}{\Xtr},\underset{k}{\ytr})=(\X_{1:k,:},\y_{1:k})$.
    If there is more than one target,
    then the labels must come from the same target,
    i.e. $\y = \Y \e_i$ for some $i=1,\ldots,m$ (where $\e_i$ is a vector with $1$ on the
    $i$-th coordinate, and zeros elsewhere).

    A \emph{prediction algorithm} is a real valued function
    $\yh(\x \,|(\X_{\mathrm{tr}},\y_{\mathrm{tr}}))$
    where $\x\in \RR^d$ is the next input vector / test vector and
    $(\X_{\mathrm{tr}}, \y_{\mathrm{tr}})$ the past training examples.
    With a slight abuse of the definition,
    we allow the function value of $\yh$ to be randomized (a random variable),
    based on some internal randomness of the algorithm.
    The accuracy of prediction $\yh$ on an example $(\x,y)$ is measured by means
    of a square loss $(y - \yh)^2$.

    \begin{remark}
    \label{d:loss}
	Our main lower bounds (Theorems \ref{t:flippm} and \ref{t:flip01})
    apply to a general class of losses $L(y,\yh)$, for which the only requirement is
    that when predicting a label $y\in\{+1,-1\}$,
    $\min_{\yh}\frac{L(-1,\yh)+L(+1,\yh)}2$ is some positive constant $c$.
    Some of our lower bounds use $y\in\{0,1\}$ instead.
    In that case we need the requirement
    $\min_{\yh}\frac{L(0,\yh)+L(1,\yh)}2$ is some
    positive constant $c'$.
    For the lower bound to apply to a neural network training with gradient descent,
    we additionally need $L(y,\yh)$
    to be differentiable in $\yh$.
    Precisely, the constants $c$ or $c'$ enter into the lower bounds.
    But this detail is a distraction and for the sake of concreteness, we state our theorems for the squared
    loss $L(y,\yh)=(y-\yh)^2$ (when $c=1$ and $c'=\sfrac 14$)
    and point out for when a linear lower bound
    still holds with the above weaker definition.
    \end{remark}
For the sake of simplicity, we avoid the discussion of randomized
algorithms in the body of the paper by assuming that the loss is square loss
and thus convex. In that case, any randomized algorithm $\yh$
can be turned into a deterministic algorithm $\yh_{\det}$:
\[
\yh_{\det}(\x| (\X_{\mathrm{tr}}, \y_{\mathrm{tr}})) =
\mathbb{E} \left[\yh(\x| (\X_{\mathrm{tr}},
\y_{\mathrm{tr}})) \right]\;
\text{(expectation w.r.t. internal randomization)},
\]
which by Jensen's inequality has loss no greater than the expected loss of
$\yh$ on any instance:
\[
(\yh_{\det}(\x| (\X_{\mathrm{tr}}, \y_{\mathrm{tr}})) - y)^2
= \left( \mathbb{E} \left[\yh(\x| (\X_{\mathrm{tr}}, \y_{\mathrm{tr}})) \right]  - y\right)^2
\le \mathbb{E}\left[ \left( \yh(\x|
(\X_{\mathrm{tr}}, \y_{\mathrm{tr}}))  - y\right)^2\right].
\]
    A prediction algorithm $\yh(\cdot\, | (\X_{\mathrm{tr}},
\y_{\mathrm{tr}}))$ is called \emph{rotation invariant}
\citep{span} if for any orthogonal matrix $\underset{d,d}{\U}$ and any input $\x\in\RR^d$:
\begin{equation}
\yh(\U \x | (\X_{\mathrm{tr}} \U^\top, \y_{\mathrm{tr}}))
= \yh(\x | (\X_{\mathrm{tr}}, \y_{\mathrm{tr}}))\, .
\label{eq:def_rot_inv}
\end{equation}
In other words, the prediction $\yh(\x, (\X_{\mathrm{tr}}, \y_{\mathrm{tr}}))$ for any input $\x$ remains the same if we rotate both $\x$ and $\X_{\mathrm{tr}}$ by matrix $\U$.
If the algorithm is randomized, then $\yh$ is a random
variable and the equality sign in definition \eqref{eq:def_rot_inv} should be interpreted as ``identically distributed''.
Our lower bounds will be for a rotation invariant subclass
of neural net algorithms.

In this paper, a {\it neural network with a fully connected input layer} is any
    real valued prediction function of the form
    $\yh= N(\underset{d}{\x}^\top\underset{d,h}{\W},\Z)$,
    where $\x$ is a real input vector, the $h$ columns of $\W$
    are the weight vectors at the $h$ input neurons,
    and $\Z$ is a fixed set of additional weights (in the upper layers).
    We only require that $N$ must be
    differentiable in $\W$. Note that the parameters $(\W,\Z)$ naturally depend on the examples $(\X_{\tr},\y_{\tr})$.

We claim that a neural network is rotation invariant, if (i) it
has a fully connected input layer in which (ii) each input node is
initialized to zero or the input layer is initialized randomly to $\underset{d,h}{\W_0}$
which has a rotation invariant distribution,
i.e. for any orthogonal matrix $\U$,
$\W_0$ and $\U \W_0$ are identically distributed.
Furthermore (iii), the input layer is updated with gradient descent on
the training set $(\Xtr,\ytr)$, and (iv) the updates of the additional weights $\Z$ depend
on $\Xtr$ only through the outputs of the bottom layer $\Xtr \W$ (as in, e.g., gradient descent updates).

To prove this claim, let $\W_t$ and $\Z_t$ be the weights of the network after initializing at $\W_0$ and $\Z_0$ and running $t$ steps of gradient descent.
We show that if
the initial weight matrix and the training set are
rotated by an orthogonal matrix $\U$, i.e.,
$\W_0 \rightarrow \U \W_0$ and $\Xtr \rightarrow \Xtr \U^\top$,
then the weights of the network undergo transformation $\W_t \rightarrow \U \W_t$ and $\Z_t \rightarrow \Z_t$ for all $t$.
The latter is straightforward from the former and the fact that $\Z_t$ depend
on $\Xtr$ only through $\Xtr \W_{j}$ $(j=0,\ldots,t-1)$, which remains invariant under $\U$:
$\Xtr \W_j \rightarrow \Xtr \U^\top \U \W_j = \Xtr \W_j$.

To show $\W_t \rightarrow \U \W_t$ by induction on $t$,
note that the gradient of the loss with respect to $\W_{t-1}$ on an example $(\x,y)$ has the form $\x \boldsymbol{\delta}^\top$, where
$\boldsymbol{\delta} = \nabla_{\a} N(\a,\Z_{t-1})$ where
$\a = \x^\top \W_{t-1}$ are the current linear activations at the input layer. Using the inductive hypothesis $\W_{t-1} \rightarrow \U \W_{t-1}$ and $\Z_{t-1} \rightarrow \Z_{t-1}$,
every such linear activation $\a$ (and thus also $\boldsymbol{\delta}$) is invariant under
$\U$ since $\a \rightarrow \x^\top \U^\top \U \W_{t-1} = \x^\top \W_{t-1} = \a$
and so the gradient undergoes a transformation $\x \boldsymbol{\delta}^\top \rightarrow \U \x \boldsymbol{\delta}^\top$. This let us conclude that $\W_t$ obtained from $\W_{t-1}$
by adding another gradients w.r.t. some examples from $(\Xtr,\ytr)$ undergoes a transformation $\W_t \rightarrow \U \W_t$.

Hence the prediction of the rotated
network on input $\x \rightarrow \U \x$ remains the same,
$$N(\x^\top\U^\top \U \W_t,\Z_t) = N(\x^\top \W_t,\Z_t),$$
and the rotation invariance follows by recalling that
$\U \W_0$ is distributed the same as $\W_0$.

\section{Lower bounds for single target, any rotation
invariant algorithm}
\label{s:single}
 Hadamard matrices are square $\pm 1$ matrices whose
 rows/columns are orthogonal.
 We assume the first row and column consists of all ones.
 We begin with a problem in which the instances are the
 rows of a $d$-dimensional Hadamard matrix and there is a single target ($m=1$),
 which is the constant function 1, i.e.  the problem
 $(\underset{d,d}{\H},\underset{d,1}{\one})$.
 This is too easy, as any algorithm which always predicts $1$ has zero loss.
 So for each sign pattern $\s\in \{+1,-1\}^d$, consider the
 problem in which the rows of $\H$ and $\one$ are sign
 flipped by $\s$, i.e. the problem $(\diag(\s)\H,\s)$.
 Note that no matter what sign pattern $\s$ we use,
 the linear weight vector $\e_1$ has zero loss on this problem,
 as $\diag(\s) \H \e_1 = \diag(\s) \one = \s$.
 Since the sign pattern is chosen uniformly,
 the algorithm is unable to learn from the labels and needs to rely on
 the information in the inputs $\diag(\s) \H$. Below
 we show that no rotation invariant algorithm is able exploit this
 information beyond the training examples and thus can't do better than
 random guessing on the unseen examples.
 See Appendix \ref{a:informal} for an informal proof for the neural net case.
   \begin{theorem}\label{t:flippm}
       For any rotation invariant algorithm receiving the first
       $k$ examples of the problem $(\diag(\s)\underset{d,d}{\H},\s)$
       where the random sign pattern $\s\in \{-1,+1\}^d$ is chosen
       uniformly, the expectation w.r.t. $\s$ of the average square loss on all $d$ examples
       is at least $1-\sfrac kd$.
   \end{theorem}

   \vspace{-1mm}
   \begin{proof}
    We first rotate the instance matrix of all the $2^d$
    sign flipped problems to a scaled identity matrix. Define a rotation matrix
    $\U_\s$ as $\sfrac{1}{\sqrt{d}}\diag(\s)\H$.
    The predictions of any rotation invariant algorithm on
    $(\diag(\s)\H,\s)$ and $(\diag(\s)\H
    \U^\top_\s,\s)=(\sqrt{d}\I,\s)$ are the same.
    Note that the rotated linear weight of the target $\s$
       is $\U_\s\e_1=\sfrac1{\sqrt{d}}\,\s$.

       Now fix $\s_{1:k}$. The algorithm receives the same first $k$
    training examples for each problem $(\sqrt{d}\I,\s)$.
       Also since $\s_{(k+1):d}$ is chosen uniformly, each of the
    $d - k$ unseen examples is labeled $\pm 1$ with equal
    probability. So the best prediction on these $d\! -\! k$
    examples is 0, incurring square loss at least 1 for each unseen
    example. We conclude that the expected average loss on all $d$
       examples is at least $1-\sfrac kd$.
    \end{proof}
    \vspace{-.5mm}
Note that in this lower bound, the algorithms have
square loss at least $\sfrac 12$ at $k=\sfrac d2$.
An alternate related proof technique, sketched in
Appendix \ref{a:dupl}, flattens the lower bound curve
using a duplication trick that moves the half point arbitrary far out.
Also in Appendix \ref{a:rand}, we strengthen the above
    theorem in two ways: the loss is allowed to be any loss
    satisfying the property specified in Remark \ref{d:loss} and the input nodes can be
    randomly initialized as long as the distributions for
    choosing the initial weight vectors is rotation
    invariant.

    Our next goal is to obtain a similar proof for the
    hardness of learning a constant function when the
    instance domain is $\{0,1\}^d$ instead of
    $\{-1,+1\}^d$. We do this because the upper bounds or
    EGU and its reparameterization with the spindly network
    require non-negative features. As for the above theorem, we
    only give a simplified version below, but this theorem
    can also be generalized using the techniques of
    Appendix \ref{a:rand}.

    The function $\sfrac{(x+1)}2$ shifts a $\pm 1$ variable
    $x$ to $0/1$. We essentially shift the previous proof
    using this transformation. For some reason which will become clear, we first
    need to remove the first row of $\H$ which is all 1.
    Let $\underset{d-1,d}{\Ht}$
    be the Hadamard matrix $\H$ with the first row removed.
    Our initial problem is now
    $\big( \sfrac 12 (\Ht + \underset{\kern-1mm d-1,d \kern-1mm}{\one}), \underset{d-1}{\one} \big)$.
    Since $\Ht\one=\zero$ all $d-1$ rows of $\Ht$ have an
    equal number of $\pm1$'s
    and therefore the rows of $\sfrac 12 (\Ht + \one)$ have an equal number of $0/1$'s.
    Instead of doing random sign flipping of the rows as we did for $\H$,
    we now complement each row of $\sfrac 12 (\Ht + \one)$ with probability $\half$.
    If $\h$ is a row of $\Ht$,
    then $\sfrac 12 (\pm\h + \one)$ are two complementary rows of bits
    that are the shifted versions of $\pm \h$.

    \begin{theorem}\label{t:flip01}
       For any rotation invariant algorithm receiving the first
       $k$ examples of the problem
       $(\sfrac 12 (\diag(\st)\Ht+\one),\sfrac 12 (\st+\one))$
       where the random sign pattern $\st\in \{-1,+1\}^{d-1}$ is chosen
       uniformly, the expectation w.r.t. $\st$ of the
       average square loss on all $d-1$ examples is at least
	$\frac 14 (1- \frac{k}{d-1})$.
   \end{theorem}

   \begin{proof}
    We rotate the instance matrices of all the
    $2^{d-1}$ randomly complemented problems to a fixed matrix.
    Define a rotation matrix $\U_\st$ as $ \sfrac{1}{\sqrt{d}} \diag([1;\st]) \H$.
    The predictions of any rotation invariant algorithm on
    the problems 	$(\sfrac 12 (\diag(\st)\Ht+\one),\sfrac 12 (\st+\one))$ and
    their rotations
  	$(\sfrac 12 (\diag(\st)\Ht+\one)\U_{\st}^\top,\sfrac 12 (\st+\one))
    =(\sfrac{\sqrt{d}}2[\underset{d-1,1}{\one},\underset{d-1,d-1}{\diag(\st)}],
    \sfrac 12 (\st+\one))$ are the same.
    Note that the rotated linear weight of the target
       $\sfrac 12 (\st+\one)$ equals $\U_{\st}\e_1 = \sfrac1{\sqrt{d}}\,[1;\st]$.

    Without loss of generality the algorithm receives the first $k$
    examples of each rotated problem during training.
    Since $\st$ is chosen uniformly, each of the
    $d-1-k$ unseen examples is labeled $0/1$ with equal
    probability. So the best prediction on these $d-1-k$
    examples is $\sfrac 12$, incurring square loss at
    least $\sfrac 14$ for each unseen
    example. Thus the expected average loss on
    all $d-1$ examples lower bounded by
    $\frac{d-1-k}{4\,(d-1)}=
    \frac 14 (1- \frac{k}{d-1}).$
   \end{proof}

  We next prove a lower bound for the Hadamard
  problem that was conjectured to be hard to learn by any
  neural network trained with gradient descent (\cite{span2}).
  The constant target is now any fixed column $\h$ of $\H$ other than
  the first one. We exploit that fact that $\h$ has
  an equal number of $\pm 1$ labels and randomly permute the
  rows of $\H$ and the target labels. The only available information
  to the rotation invariant algorithm are the $k$ labels seen in the training sample.
  Therefore, the best prediction the algorithm can be
  deduced from the average count of the remaining labels. On expectation,
  this gives a lower bound of $1-\frac{k}{d-1}$ on the
  average square loss, which is only slightly weaker than the bound of Theorem \ref{t:flippm}.

   \begin{theorem}\label{t:permute}
       For any rotation invariant algorithm receiving the first $k$ examples
       of the problem $(\PPi \H, \PPi \h)$,
       where $\h  = \H \e_i$ for some fixed $i \in \{2,\ldots,d\}$
       is the $i$-th column of $\H$,
       and $\PPi$ is a permutation matrix chosen uniformly at random,
       the expectation w.r.t. $\PPi$ of the average square
       loss on all $d$ examples is at least $1-\frac{k}{d-1}$.
   \end{theorem}

   \begin{proof}
    Define a rotation matrix $\U$ as $\sfrac{1}{\sqrt{d}}\,\PPi \H$.
    The predictions of any rotation invariant algorithm on
    $(\PPi \H, \PPi \h)$ and $(\PPi \H \U^\top\!, \PPi \h) =
    (\sqrt{d}\I,\PPi \h)$ are the same.
    The algorithm receives the first $k$ examples of each problem $(\sqrt{d} \I, \PPi \h)$
    as training examples. Since the problems are permuted uniformly,
    every permutation of the unseen $d-k$ labels is equally likely.
    So the best prediction on the $d-k$ unseen labels is
       the average count of the unseen labels (which the algorithm
    can deduce from the count of the seen labels and the
       fact that $\h$ has an equal number of $\pm 1$ labels.)
       If we denote the number of unseen $+1/\!-1$ labels as $\ELL$
    and $(d-k)-\ELL$, respectively, then the best prediction on all unseen
    labels is $\yh = \frac{\ELL - (d-k-\ELL)}{d-k}$.
    The total square loss on all unseen examples is thus at least:
    \[
      \ELL (1-\yh)^2 + (d-k-\ELL) (-1 - \yh)^2
       = \frac 4{d-k} \;\ELL(d-k-\ELL).
    \]
    In Appendix \ref{app:permutation_theorem} we show that $\ELL$ follows a hyper geometric
    distribution, by which we can compute the expected total loss as
    \[
    \frac{4}{d-k} \;\mathbb{E}[\ELL(d-k-\ELL)] = d - \frac{kd}{d-1}.
    \]
    Lower bounding the loss of the algorithm on seen examples by $0$, and dividing by $d$,
    the expected average loss over all $d-1$ examples is at least $1-\frac{k}{d-1}$.
\end{proof}

\subsection{Lower bound for random Gaussian inputs}

The lower bounds in the previous section result from feeding a rotation invariant
algorithm examples which are all orthogonal to each other. These type of problems
might seem somewhat specialized and therefore in this section we show that
rotation invariant algorithms are also unable to
efficiently learn a simple class of problems in which
the entries of $\X$ follow an i.i.d. standard Gaussian distribution,
while the labels are generated noise-free by a sparse weight vector.
On the other hand, this class of simple problems is perfectly learnable
from just a single training example
by a straightforward (non rotation invariant) algorithm.

We thus consider the problem matrix $(\X, \y)$,
where the rows of $\X$ (instances) are generated
i.i.d. from $N(\boldsymbol{0}, \I)$,
while $\y = \X \w^{\star}$
for some $\w^{\star}$ with $\|\w^{\star}\|=1$. Note that the setup is noise-free as
$\w^{\star}$ has zero loss on any such $(\X, \y)$.
The standard Gaussian distribution is rotation invariant: if
$\x \sim N(\boldsymbol{0}, \I)$ then $\U \x \sim N(\boldsymbol{0}, \I)$
for any orthogonal matrix $\U$. In fact, the whole analysis in this section
would equally well apply to \emph{any} rotation invariant distribution with covariance
matrix $\I$, but we keep the Gaussian distribution for the sake of simplicity.

\begin{theorem}
\label{thm:spherically_symmetric}
       For any rotation invariant algorithm receiving the first
       $k$ examples of the problem
$(\X, \y \!=\!\! \X\! \w^{\star})$ for any $\|\w^{\star}\|=1$ with the $d$ rows of
$\X$ being generated i.i.d. from $N(\boldsymbol{0}, \I)$,
       the expectation w.r.t. $\X$ of the average square loss on all $d$ examples
       is at least $(1 - \sfrac kd)^2$.
\end{theorem}
This lower bound $(1 - \sfrac kd)^2=1-2\,\sfrac kd+(\sfrac kd)^2$
is slightly weaker than the ones for orthogonal instances.
Yet the algorithm has expected loss at least $\sfrac 14$ after seeing half of the examples.

\begin{proof}(sketch, full proof in Appendix \ref{app:Gaussians_proof})
We first show that a rotation invariant algorithm has the same expected average
loss for any $\w^{\star}$ with $\|\w^{\star}\|=1$, as both the data
and $\w^{\star}$ can be rotated without changing the data distribution or
the predictions of the algorithm. Then we consider a problem
in which $\w^{\star}$ itself is drawn uniformly from a unit sphere,
and show that \emph{every} algorithm (not necessarily rotation invariant) has
expected average loss at least $(1 - \sfrac kd)^2$ on all
    examples where the expectation is
taken with respect to a random choice of $\w^{\star}$ and $\X$.
It then follows from the first argument that a rotation invariant algorithm
has expected average loss at least $(1 - \sfrac kd)^2$ for \emph{any} choice of $\w^{\star}$.
\end{proof}
Since the loss of a rotation invariant algorithm does not depend on the choice of the
target weight vector $\w^{\star}$,
one can choose any fixed sparse vector $\w^{\star} = \e_i$ for
$i \in \{1,\ldots,d\}$ as the target.
In each case the algorithm still incurs expected loss at
least $(1-\sfrac{k}{d})^2$ on average after seeing $k$ examples.
On the other hand for a Gaussian input matrix, sparse weight vectors are
easy to learn from a \emph{single} example.
This is because for $\w^{\star} = \e_i$, the label
coincides with the $i$-th input feature and
each entry of every row of $\X$ is unique with probability one.

We note that the lower bound of Theorem \ref{thm:spherically_symmetric} is tight and
matched by the (rotation invariant) least squares algorithm (proof in Appendix
\ref{app:LS}).
\begin{theorem}
\label{thm:least_squares_optimal}
Consider the problem $(\X, \y = \X \w^{\star})$ for any
    $\|\w^{\star}\|=1$ where the columns of
$\X$ generated i.i.d. from $N(\boldsymbol{0}, \I)$.
The least squares algorithm, which upon seeing $k$ examples, predicts
with the weight vector
$\w = \X_{1:k,:}^{\dagger} \y_{1:k}$,
where $\A^{\dagger}$ denotes the pseudo-inverse of $\A$,
has expected average loss $(1 - \sfrac kd)^2$ on all $d$ examples,
where the expectation is w.r.t. the random choice of $\X$.
\end{theorem}

\section{Shortcomings of single target lower bound technique
of Theorem \ref{t:flippm}}

    \begin{enumerate}
\item
    The rotation invariant initialization of the fully connected
	    input nodes is necessary:
    If all input nodes were initialized to
    $\e_1$, then the network would learn with zero examples because
    $$ \underbrace{s_i \h_i}_{\kern-5mm \text{sign flipped $i$-th instance} \kern-5mm}
	    \cdot \;\;\;\underbrace{\e_1}_{\w_0} =s_i.$$
\item
    The lower bound is also broken if the instances
    are \emph{embedded} by an arbitrary feature map
    $\mathbb{R}^d \ni \x \mapsto \phi(\x) \in \mathbb{R}^m$,
    as allowed by kernel methods. For example, all rows of $\H$ can be embedded as
    the same row, $\phi(\h_i)=\e_1$.
    Now after the fully connected input nodes all receive the sign flipped
    first instance $s_1 \e_1$ as their input,
    we are essentially in the target initialized case.
\item
    Most of all the lower bounds shown so far mostly require the instances to be
    orthogonal. We already showed that random Gaussian features
	    lead to a slightly weaker lower bound for
	    sparse targets (Theorem \ref{thm:spherically_symmetric}). However
	    ideally we want lower bounds when the features
	    are random $\pm 1$.
    \end{enumerate}
    In the next section we show that in a special case,
    the initialization $\w_0=\zero$ can be shown to be
    optimal. We then expand the SVD based lower bounding technique of \cite{span}
    which covers the case of arbitrary initialization and random $\pm 1 $ features.
    However so far, all of this only works for linear neural networks
    with up to two fully connected layers.
\section{Zero initial weight vector is optimal for a single linear neuron}

\begin{theorem}
    Consider learning the problem
    $(\H,[\one,-\one])$ on a single linear neuron with any
    initialization $\w_0$. Assume the algorithm updates
    with gradient descent based on the first $k$ instances.
    Then the average square loss over all $d$ instances and
    both targets is at least $1-\sfrac kd$ and is minimized
    when $\w_{0,k+1:d}=\zero$.
\end{theorem}
\begin{proof}
   Since the algorithm is rotation invariant%
    \footnote{To be more precise,
   the rotation invariance is broken by an arbitrary intialization $\w_0$,
   but the algorithm's prediction are the same after rotating the data, if
   we also rotate $\w_0$. Since $\w_0$ is arbitrary, we can as well consider
    its rotated version as the initialization.}, it is
    convenient to switch to proving the lower bound for problem matrix
    $(\sqrt{d}\,\I,[\one,-\one])\text{  and any }\w_0$.
    For both targets, the algorithms weight vector is $\w_0$ plus a linear
    combination of the first $k$ instances $\sqrt{d}\e_i$, i.e.
    $\w=\w_0+[\underset{1, k}{\b},\underset{1, d-k}{\zero}]^\top$,
    where $\b$ depends on the target.
    For any unseen instance $\sqrt{d}\e_j$ for $j>i$,
    the linear neuron predicts $\sqrt{d}\e_j^\top \w= \sqrt{d} w_{0,j}$
    on both targets when the label is $\e_j^\top (\pm \one) =\pm 1$.
    So the average loss is at least
    \begin{align*}
	&\sfrac1{2d}\sum_{j=k+1}^d
	\underbrace{\Big((\sqrt{d}w_{0,i}-1)^2+ (\sqrt{d}w_{0,i}+1)^2\Big)}
	 _{\text{minimized when $w_{0,j}=0$}}
	\;\ge\; \frac1{2d} \;2(d-k) \;=\;1-\frac{k}{d}.
    \end{align*}
    This proves that the initialization $w_{0,j}=0$ is optimal for all $j>k$.
    \end{proof}
    This type of lower bounds does not work for a single target (say $\one$)
    because $\w_0$ could be set to generate this target and no loss
    would be incurred. By additionally permuting the rows
    of $\H$, we can force all components of $\w_0$ to 0.

\section{SVD based lower bounds for linear neural nets}
\label{s:sing}
An alternative second technique for deriving lower bounds
for learning the Hadamard problem using GD uses the
SVD spectrum of the design matrix \citep{span}.%
\footnote{See \cite{pritish} for a more recent paper based
on similar lower bound techniques.}
The main idea is to lower bound the loss on all
examples averaged over all targets based on the
tail sum of the squared singular values of a matrix $\Y$
containing the label vectors of all targets as columns.
The instances are the transposed rows of $\underset{n, d}{\Y}$,
transformed by an arbitrary feature map $\phi\!:\!\RR^d \!\!\rightarrow\! \RR^m.
\!$
We first present this technique in its simplest form.

\begin{theorem}[\citep{span}]
    \label{thm:span}
For any $\phi$ map, the average square loss for learning
    the problem $(\underset{n,m}{\phi(\Y)},\underset{n,d}{\Y})$ after seeing any $0 \leq k \leq \rank(\Y)$
    training instances and using GD on a single linear neuron with zero initialization
    is lower bounded by
  $\frac1{\Vert\Y\Vert_{\text{F}}^2} \sum_{i=k+1}^{\rank(\Y)} s_i^2,$
  where $s_i$ is the $i$-th singular value of $\Y$ in the decreasing order.
\end{theorem}
\begin{proof}
    For a fixed target column $\y_i$ of $\Y$ and after seeing $k$ arbitrary rows
    $\underset{k, m}{\phi(\widehat{\Y})}$,
    the weight vector of the GD algorithm can be written as
    a linear combination of these rows, i.e. $\w =
    \phi(\widehat{\Y})^\top \a_j$ for some $\a_j \in \RR^k$. We can now write the average square loss as
\[ \frac1{\Vert\Y\Vert_{\text{F}}^2}\sum_{j\in [d]}
    \Vert \phi(\Y) \phi(\widehat{\Y})^\top \a_j - \y_j\Vert^2
    = \frac1{\Vert\Y\Vert_{\text{F}}^2}\,
    \Vert \overbrace{\phi(\Y)
    \underbrace{\phi(\widehat{\Y})^\top\!}_{m,k}
    \underset{k, d}{\A}}^{\text{rank-}k} - \Y\Vert_{\text{F}}^2
    \geq \frac1{\Vert\Y\Vert_{\text{F}}^2} \sum_{i=k+1}^{\rank(\Y)} s_i^2\, ,
\]
where $\A$ is the matrix with $j$-th column equal to $\a_j$.
\end{proof}
Notice that the above relies on the fact that for any rank $0 \leq k \leq \rank(\Y)$, the square loss of any rank $k$ approximation
of matrix $\Y$ is lower bounded by
the sum of the smallest $\rank(\Y)-k$ squared singular values
of $\Y$. As a result, the bound is effective only when $\Y$
has a flat spectrum. The original bound of
Theorem~\ref{thm:span} was proposed for learning the
$d$-dimensional Hadamard matrix $\Y=\H$ with $n=d$. In this
case $\Vert\H\Vert_{\text{F}}^2=d^2$, $\rank(\H)=d$, and
all $d$ singular values $s_i^2$ are equal to $d$. Thus the
lower bound reduces to $\frac1{d^2} (d-k) d = 1 - \frac{k}{d}$,
for $0 \le k \le  d$.
When $\Y$ contains random $\pm 1$ features, then the
tail of a square spectrum is tightly concentrated and the
lower bound becomes $1\! -\! \frac{ck}{d}$ where $c\!>\!0$
(Appendix~\ref{a:randpm}).

We next observe that with any (fixed) initialization
$\w_0$ of the single neuron, the rank of
weight matrix goes up by one and the lower bound becomes
$\frac1{\Vert\Y\Vert_F^2} \sum_{i=k+1+1}^{\rank(\Y)} s_i^2$.
We show this for a doubled version of the Hadamard problem, with
targets $\pm \bm{h}_j$ for $j \in [n]$, because this helps with
the proof of the next theorem.
\begin{theorem}
  \label{thm:span-hadamrd}
  For any $\phi$ map, the average square loss for learning
    the problem $(\phi(\H),[\H,-\H])$ after seeing any $0 \leq k \leq d$
    training instances and using GD on a single linear
    neuron with any initialization
    is lower bounded by $1 - \frac{k+1}{d}$.
\end{theorem}
\begin{proof}
  Similar to the proof of Theorem~\ref{thm:span}, we have
\begin{align*}
    \frac{1}{2d^2}&\Vert\!\overbrace{\phi(\H)(\underset{m,k}{\phi(\widehat{\H})^\top}\!\! \A
    +\underset{m,1}{\w_0}\,\underset{1,2d}{\one^\top})}^{\text{rank $k+1$}}
    -[\H,\!-\H]\Vert_{\text{F}}^2
     =\frac1{2d^2}\!\!\!\sum_{i=k+2}^d \!\!\!s_i^2
  =\frac{(d\!-\!(k\!+\!1))(2d)}{2d^2} =1\!-\!\frac{k\!+\!1}{d}\, .
 \vspace{-5mm}
\end{align*}

\vspace{-10.5mm}\end{proof}
The above lower bound is not directly comparable with the
available upper bounds for the
spindly network which require the input features and labels to be in a the range $[0,X]$
and $[0,Y]$, respectively, for some constants $X,Y > 0$.
However we can shift the $\pm 1$ features of $\Y$ to $0/1$ as was done for
Theorem \ref{t:flip01}. The spectrum essentially remains flat.

Essentially we need to compute the SVD spectrum of
$\underset{d, 2d}{\M}\coloneqq([\H,-\H]\, +
\underset{d,2d}{\one})/2$. Matrix $\M$ has rank at least $d$, thus we have $\rank(\M)=d$ and all rows of $\M$ have equal number ($d$ many) $1$'s and $0$'s. Thus $\Vert\M\Vert_{\text{F}}^2=d^2$. Note that we have
\begin{align*}
    \M\,\M^\top&=([\H,-\H]+\underset{d,2d}{\one})/2 \quad
    ([\H,-\H] +\underset{d,2d}{\one})^\top/2 =
    (2d\,\underset{d,d}{\I} + 2\, \underset{d,d}{\zero}
    +2d \,\underset{d,d}{\one})/4
    = \frac d2 (\underset{d,d}{\I}+\underset{d,d}{\one}).
\end{align*}
Recall that the $s_i^2$ of $\M$ are the eigenvalues of
$\M\,\M^\top$ sorted in decreasing order. The following
shows the eigensystem of $\M\,\M^\top$ is $\frac{\H}{\sqrt{d}}$ and also gives its eigenvalues:
\begin{align*}
\frac{\H^\top}{\sqrt{d}} \;
    \frac d2 (\underset{d,d}{\I}+\underset{d,d}{\one}) \;
\frac{\H}{\sqrt{d}}
    &= \H^\top(\underset{d,d}{\I}+\underset{d,d}{\one})\H / 2
    =(\H^\top+d\;[\one;\underset{d-1,1}{\zero}]^\top) \; \H/2\\
    &=\frac d2 \,\underset{d,d}{\I}+ \frac{d^2}2 \diag([1,\underbrace{0,\ldots,0}_{d-1}])
=\diag([\frac{d^2}2\!+\!\frac d2,\underbrace{\frac d2, \ldots, \frac d2}_{d-1}])\, ,
\end{align*}
where we use the fact that the first row and column of $\H$ has all ones and all remaining rows and columns have an equal number of $\pm 1$'s. The following lower bound carries immediately.
\begin{corollary}
  \label{cor:shifted-hadamard}
   For any $\phi$ map, the average square loss for learning
    the problem $(\phi(\H),([\H,$\newline$-\H]+\one)/2)$ after seeing any $0 \leq k \leq d$
    training instances and using GD on a single linear
    neuron with any initialization is lower bounded by
    $\frac1{2d^2} (d-k-1) \frac d2 = \frac14-\frac {k+1}{4d}.$
\end{corollary}
We now consider a two-layer network with fully connected linear layers having weights
$\underset{m,h}{\W^{(1)}}$ and $\underset{h,1}{\w^{(2)}}$,
respectively, where $m$ is the input dimension and $h$ is the number of hidden units.
Assume we are learning the problem
$(\X\!=\!\phi(\Y),\Y)$ with gradient descent based on
the first $k$ training examples $(\Xtr\!=\!\X_{1:k,:},\ytr\!=\!\y_{1:k})$.
\begin{theorem}
  \label{thm:two-layer}
    After seeing $k$ training examples $(\Xtr,\ytr)$,
    GD training keeps the combined weight $\W^{(1)}\w^{(2)}$ in the
    span of $[\W_0^{(1)}\w_0^{(2)},\W_0^{(1)}(\W_0^{(1)})^\top\Xtr^\top,\Xtr]$,
    which is rank $\le 2k+1$.
\end{theorem}
	    The above rank bound (proven in Appendix~\ref{app:two-layer})
	    immediately gives a lower bound of $1-\frac{2k+1}d$
	    for $(\phi(\H),\H)$. We conjecture the lower bound
	    can be improved to $1-\frac{k+1}d$ (which is
	    observed experimentally).
            Since the rank bound of $2k+1$ can be tight, we
	    see that the rank argument gets weak.
For three layers the rank goes up too quickly for this
methodology to be useful.

\section{Conclusion}
We show in this paper that GD trained neural networks with a fully connected input layer
and rotation invariant initialization at the input nodes cannot
sample efficiently learn single components.
Most of our lower bounds use the rows of Hadamard matrices
as instances which use $\pm 1$ features. We show a slightly
weaker but still linear lower bound when the features are
Gaussian and conjecture that the same weaker lower bound also holds
for random $\pm 1$ features.\footnote{%
Note that from random matrix theory we know that the spectral properties of random Gaussian and random $\pm 1$ matrices
are known to be essentially the same \citep{matrix_th}.}.

It would also be interesting to further investigate the
power of fully connected linear networks.
We made some progress by extending the SVD based method of
\cite{span} which provides linear lower
bounds for learning the components of random $\pm 1$ matrices even
when the instances are transformed with an arbitrary $\phi$ map
and when an arbitrary initialization is used.
For this method it is necessary to
bound the rank of a certain weight matrix, but we were only
able to do this for up to two fully connected layers.
Note that all lower bounds on GD trained neural nets with complete
input layers are for a fixed target component when no transformation
is used. We conjecture that the same linear lower bounds
hold for arbitrary transformations if we average over target
components. However, totally different proof techniques
would be needed to prove this conjecture.

Finally, there are many technical open problems regarding the GD
training of the spindly network (Figure \ref{fig:spindly}):
Do the regret bounds still hold when the domain is $[-1,1]^d$
instead of $[0,1]^d$? Is clipping of the predictions
necessary? Experimentally the network still
learns efficiently if all $2d$ weight are initialized with
random positive numbers or if the $d$ bottom weights are
initialized to 0. However the regret bounds of
\citep{wincolt20} do not hold for these cases.

\newpage
\paragraph{Acknowledgments}
Thanks to Vishy Vishwanathan for many inspiring
discussions and for letting us include the concentration theorem of Appendix
\ref{a:randpm} which is due to him. This research was partially
support by NSF grant IIS-1546459.

\bibliography{../refs}

\begin{thebibliography}{16}
\providecommand{\natexlab}[1]{#1}
\providecommand{\url}[1]{\texttt{#1}}
\expandafter\ifx\csname urlstyle\endcsname\relax
  \providecommand{\doi}[1]{doi: #1}\else
  \providecommand{\doi}{doi: \begingroup \urlstyle{rm}\Url}\fi

\bibitem[Amid and Warmuth(2020{\natexlab{a}})]{regretcont}
Ehsan Amid and Manfred~K. Warmuth.
\newblock Reparameterizing mirror descent as gradient descent.
\newblock \emph{arXiv preprint arXiv:2002.10487}, 2020{\natexlab{a}}.
\newblock To appear Advances in Neural Information Processing Systems
  (NeurIPS).

\bibitem[Amid and Warmuth(2020{\natexlab{b}})]{wincolt20}
Ehsan Amid and Manfred~K Warmuth.
\newblock Winnowing with gradient descent.
\newblock In \emph{Conference on Learning Theory (COLT)}, pages 163--182. PMLR,
  2020{\natexlab{b}}.

\bibitem[Bogdanov et~al.(2019)Bogdanov, Sabin, and Vasudevan]{xor}
Andrej Bogdanov, Manuel Sabin, and Prashant~Nalini Vasudevan.
\newblock {XOR} codes and sparse learning parity with noise.
\newblock In \emph{Proceedings of the 2019 Annual ACM-SIAM Symposium on
  Discrete Algorithms}, pages 986--1004, 2019.

\bibitem[Davidson and Szarek(2003)]{DavSza03}
K.~R. Davidson and S.~J. Szarek.
\newblock Local operator theory, random matrices and banach spaces.
\newblock In J.~Lindenstrauss and W.~Johnson, editors, \emph{Handbook of the
  Geometry of Banach Spaces}, volume~1, chapter~8, pages 317--366.
  North-Holland, Amsterdam, 2003.

\bibitem[Derezinski and Warmuth(2014)]{span2}
Michal Derezinski and Manfred~K. Warmuth.
\newblock The limits of squared {E}uclidean distance regularization.
\newblock In \emph{Advances in Neural Information Processing Systems
  (NeurIPS)}, pages 2807--2815, 2014.

\bibitem[Gunasekar et~al.(2017)Gunasekar, Woodworth, Bhojanapalli, Neyshabur,
  and Srebro]{srebro1}
Suriya Gunasekar, Blake~E Woodworth, Srinadh Bhojanapalli, Behnam Neyshabur,
  and Nati Srebro.
\newblock Implicit regularization in matrix factorization.
\newblock In \emph{Advances in Neural Information Processing Systems
  (NeurIPS)}, pages 6151--6159, 2017.

\bibitem[Kamath et~al.(2020)Kamath, Montasser, and Srebro]{pritish}
Pritish Kamath, Omar Montasser, and Nathan Srebro.
\newblock Approximate is good enough: Probabilistic variants of dimensional and
  margin complexity.
\newblock In \emph{Conference on Learning Theory (COLT)}, pages 163--182. PMLR,
  2020.

\bibitem[Kivinen and Warmuth(1997)]{eg}
Jyrki Kivinen and Manfred~K. Warmuth.
\newblock Exponentiated gradient versus gradient descent for linear predictors.
\newblock \emph{Information and Computation}, 132\penalty0 (1):\penalty0 1--63,
  1997.

\bibitem[Maass and Warmuth(1998)]{MW98}
M.~Maass and M.K. Warmuth.
\newblock Efficient learning with virtual threshold gates.
\newblock \emph{Information and Computation}, 141\penalty0 (1):\penalty0
  66--83, February 1998.

\bibitem[Meckes(2004)]{Meckes04}
M.~W. Meckes.
\newblock Concentration of norms and eigenvalues of random matrices.
\newblock \emph{Journal of Functional Analysis}, 211\penalty0 (2):\penalty0
  508--524, June 2004.

\bibitem[Safran and Shamir(2017)]{shamir}
Itay Safran and Ohad Shamir.
\newblock Depth-width tradeoffs in approximating natural functions with neural
  networks.
\newblock In \emph{International Conference on Machine Learning (ICML)}, pages
  2979--2987, 2017.

\bibitem[Sylvester(1867)]{sylvester}
J.J. Sylvester.
\newblock Thoughts on inverse orthogonal matrices, simultaneous
  signsuccessions, and tessellated pavements in two or more colours, with
  applications to {N}ewton's rule, ornamental tile-work, and the theory of
  numbers.
\newblock \emph{The London, Edinburgh, and Dublin Philosophical Magazine and
  Journal of Science}, 34\penalty0 (232):\penalty0 461--475, 1867.

\bibitem[Tao(2012)]{matrix_th}
Terence Tao.
\newblock \emph{Topics in Random Matrix Theory}.
\newblock American Mathematical Society, 2012.

\bibitem[Telgarsky(2016)]{matus}
Matus Telgarsky.
\newblock Benefits of depth in neural networks.
\newblock In \emph{Conference on Learning Theory (COLT)}, pages 1517--1539,
  2016.

\bibitem[Warmuth et~al.(2014)Warmuth, Kot{\l}owski, and Zhou]{rotinv}
M.~K. Warmuth, W.~Kot{\l}owski, and S.~Zhou.
\newblock Kernelization of matrix updates.
\newblock \emph{Journal of Theoretical Computer Science}, 558:\penalty0
  159--178, 2014.
\newblock Special issue for the 23nd International Conference on Algorithmic
  Learning Theory (ALT).

\bibitem[Warmuth and Vishwanathan(2005)]{span}
M.K. Warmuth and S.V.N. Vishwanathan.
\newblock Leaving the span.
\newblock In \emph{Proceedings of the 18th Annual Conference on Learning Theory
  (COLT)}, 2005.

\end{thebibliography}

\appendix
\section{Informal proof of Theorem \ref{t:flippm} for the
case of GD trained neural nets with a complete input layer}
\label{a:informal}
    The problem $(\H,\one)$, i.e. learning the constant target 1 when the
    instances are the (transposed) orthogonal rows of the Hadamard matrix
    $\underset{d,d}{\H}$, is too easy.
    So we feed the net randomly sign flipped rows
    $(s\h^\top,s)$, for $s=\pm 1$, where $\h$ is a row of
    $\H$ in the training set.
    The neural net has a complete input layer that is trained with gradient descent.
    So the weights at the input nodes are linear
    combinations of the $k$ seen instances.
    Thus the input nodes are ``no help'' for
    predicting the labels of the unseen examples
    since they are orthogonal to the seen ones.
    Also the labels of the seen examples are random and
    thus do not contain any information for predicting the
    label of the unseen examples. The unseen sign flipped instances are
    labeled $\pm 1$ with equal probability. So the optimum prediction of
    the unseen $d-k$ examples is 0, and we incur one unit of
    square loss for each. The loss on the seen
    examples is non-negative and thus the total average loss
    on all $d$ examples is at least $\frac{d\!-\!k}d\!  =\!1\!-\!\frac{k}{d}$.
    \Qed

\section{Alternate proof style related to Theorem \ref{t:flippm} that
interleaves duplicated positive and negative rows of Hadamard.}
\label{a:dupl}
Instead of using sign flipped rows of $\H$, we use the following instance matrix:
Matrix $\X_q$ (for $q\ge 1$) consists of $q$ copies of the first row of
$\H$, followed by $q$ copies of minus the first row,
$q$ copies of the second, $q$ copies of minus the second,
$\ldots$, for a total of $2q\,d$ rows.
The label vector $\y_q$ is again the first column of $\X_q$, i.e. it consists of
alternating blocks of $(+1)^q$ and $(-1)^q$.
After receiving the first $k$ examples of $(\X_q,\y_q)$, LLS has average loss
\begin{equation}
    \label{e:wor}
    1-\frac{2q\lceil{\frac{k}{2q}}\rceil}{2q\,d}\, ,
\end{equation}
on all $2q\,d$ examples because it predicts 0 on all unseen block pairs.
(See saw tooth curve in Figure \ref{f:saw} in which each
tooth drops by $\sfrac 1d$ and then stays flat for $2q-1$ steps).
Also if each pair $1\le i \le d$ of
blocks is randomly swapped, i.e.
the $q$ rows of minus row $i$ come before
the $q$ rows of plus row $i$, then after receiving $k$
examples,
the expected average square loss is
lower bounded by the same saw tooth curve,
where the expectation is w.r.t. $d$ random swaps.
\begin{wrapfigure}{r}{0.4\textwidth}
        \vspace{-3mm}
\begin{center}
    \vspace{-2mm}
    \includegraphics[width=.4\textwidth]{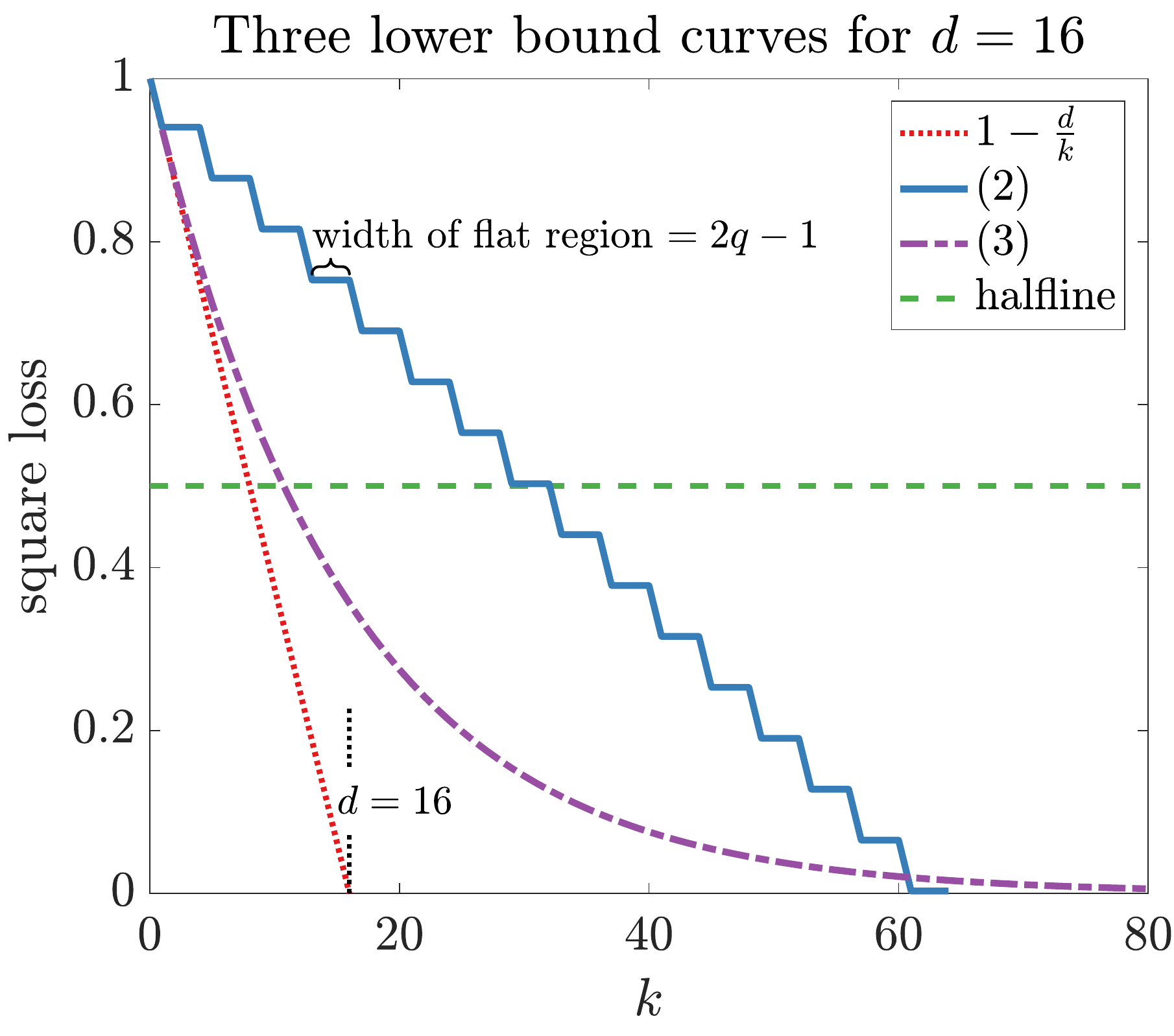}
    \vspace{-0.8cm}
    \caption{Bound of Thm. \ref{t:flippm} (dotted) \& bounds
    for sampling from $(\X_2,\!\y_2)$
    w.o. replacement (solid) and
    i.i.d. (dash-dot).}
\label{f:saw}
\end{center}
\vspace{-0.5cm}
\end{wrapfigure}
This is because in that case the optimal prediction for any
GD trained neural network on all unseen blocks is 0 (proof not shown).
This $q$ ``duplication trick'' moves the point where the loss
is half to $k=q\,d$ and the loss 0 point to $k=2q\,d$ (See
Figure \ref{f:saw}).

Also note that training on the first $k$ examples of random permutations of $(\X_q,\y_q)$
for large $q$ approximates sampling with replacement on
this type of problem. Again the optimal algorithm must
predict zero on all $i$, s.t. row $i$ nor minus row $i$
have not been sampled yet. We believe that our sampling
w.o. replacement lower bounds are more succinct, because
they avoid the coupon collector problem for collecting the $d$ row pairs:
For the sake of completeness, the expected square loss after
picking $k$ i.i.d. examples from $(\X_1,\y_1)$ is
\begin{equation} \label{e:iid}
    (1-\sfrac 1d)^k,
\end{equation}
which is the expected number of pairs not collected yet
(See plot in Figure \ref{f:saw} for comparison).

\section{Proof of Theorem \ref{t:flippm}
for losses that only require the property from Remark \ref{d:loss}}
\label{a:rand}

Here we extend Theorem \ref{t:flippm} to arbitrary losses
which only satisfy the property in Remark \ref{d:loss}. Since we do not assume
 convexity of the loss, we need to consider a randomized rotation invariant algorithm
(as the randomized algorithms are no longer dominated by deterministic ones). Formally,
let $\t$ denote a collection of random variables (independent of the data),
which represent the internal randomization of the algorithms.
Given the training sample $(\X_{\mathrm{tr}},\y_{\mathrm{tr}})$
and the input $\x$, the prediction of the algorithm is a random variable
$\yh^{}_{\t}(\x |(\X_{\mathrm{tr}},\y_{\mathrm{tr}}))$.
The algorithm is rotation invariant if for any orthogonal matrix
$\U$,
\begin{equation}\label{e:rot}
\text{
$\yh^{}_{\t}(\x|  (\X_{\mathrm{tr}},\y_{\mathrm{tr}}))$
and $\yh^{}_{\t}(\U \x| (\X_{\mathrm{tr}}\U^\top,\y_{\mathrm{tr}}))\, ,$}
\end{equation}
are identically distributed.

In the context of neural networks, the random variables $\t$
consist of the random initialization $\W_0$ as well as
additional random variables used by the algorithm.
We assume that $\t$ is not affected by rotating the instances.
Also as in the deterministic algorithm case,
if the weights $\W$ at the bottom layer are updated with
gradient descent, then rotating the instances causes the
weights $\W$ to be counter rotated and the predictions stay unchanged.

We now continue with the general proof using assumption \eqref{e:rot}
for the case of Theorem \ref{t:flippm}.
Take a problem $(\diag(\s) \H, \s)$ with the sign pattern $\s \in \{-1,+1\}^d$ chosen
uniformly at random, and assume the algorithm has received its first $k$ examples,
$(\X_{1:k,:},\s_{1:k})$.
Define a rotation matrix $\U_{\s} = \sfrac{1}{\sqrt{d}}\;\text{diag}(\s) \H$.
The expected total loss of the algorithm on the unseen examples is
\begin{align*}
\mathrm{expected~loss} &=
\mathbb{E}_{\s}\mathbb{E}_{\t}\left[ \sum_{t=k+1}^d
    L\big(s_t, \yh^{}_{\t}(\x_t | (\X_{1:k,:},\s_{1:k}))\big) \right] \\
&\stackrel{(*)}{=}
\mathbb{E}_{\s}\mathbb{E}_{\t}\left[ \sum_{t=k+1}^d
    L\big(s_t, \yh^{}_{\t}(\U_{\s} \x_t | (\X_{1:k,:}\U_{\s}^\top,
\s_{1:k}))\big) \right] \\
&=
\mathbb{E}_{\s}\mathbb{E}_{\t}\left[ \sum_{t=k+1}^d
    L\big(s_t, \yh^{}_{\t}(\sqrt{d} \e_t | (\sqrt{d} \I_{1:k,:},
\s_{1:k}))\big) \right] \\
&= \mathbb{E}_{\t} \mathbb{E}_{\s_{1:k}}\mathbb{E}_{\s_{(k+1):d}|\s_{1:k}}
\left[ \sum_{t=k+1}^d
    L\big(s_t, \yh^{}_{\t}(\sqrt{d} \e_t | (\sqrt{d} \I_{1:k,:},
\s_{1:k}))\big) \right]\\
&\ge \mathbb{E}_{\t} \mathbb{E}_{\s_{1:k}}
\left[ \frac{1}{2}\sum_{t=k+1}^d
\min_{\yh \in \mathbb{R}} \left\{ L\big(-1, \yh) + L \big(+1, \yh) \right\} \right] \\
&= (d-k) \min_{\yh \in \mathbb{R}} \left\{ \frac{1}{2}(L\big(-1, \yh) + L \big(+1, \yh)) \right\}
\ge (d-k) c,
\end{align*}
where in $(*)$ we used the rotation invariance of the algorithm and the last inequality
follows from Remark \ref{d:loss}. Lower bounding the loss of the algorithm
on seen examples by zero, the average expected loss of the rotation invariant algorithm
is at least $\sfrac{1}{d}\,(d-k)c = (1-\sfrac{k}{d})c$. \Qed
\bigskip
It is worthwhile to check if the lower bounds for Hadamard data also hold for other type
of randomization beyond rotation invariance.
As an interesting example, consider reflective invariant distributions.
The distribution $P$ is
{\it reflective invariant} if $\forall \; \t, P(\t)=P(-\t)$.
The question is whether a lower bound can be proven for
a neural network trained with a gradient descent method, when the random initialization
of the weights is reflective invariant.
Unfortunately this is not the case:
Let the randomness $\t$ correspond to a random intialization of the weight
vector
$\w_0 = (s, 0, \ldots,0)$, with $s \in \{-1,1\}$ generated uniformly
at random (so\, $\w_0 = s \w^{\star}$). This is a
reflective invariant initialization.
On the first training example $\x_1$, the bottom neurons evaluate to $\w_0^\top \x_1 = s y_1$, so the upper layers (getting $y_1$ as feedback) can in principle learn to multiply the result by $s$ again to get $\yh_1 = y_1$.
On unseen examples $\x_i$, the bottom neurons again
evaluate to $\w^\top \x_i = \w_0^\top \x_i = s y_i$, and
multiplying by $s$ by upper layers gives zero loss.

\section{Hadamard matrix as an exponential expansion}
\label{a:hadam}
The simplest way to construct Hadamard matrices of
dimension
$d=2^q$ is to use the following recursive construction
credited to Sylvester~\citep{sylvester}:
$$\H_1=[1],\quad
\H_2=\begin{bmatrix}
1&1\\1&-1\end{bmatrix},\ldots,\quad
\H_{q+1}=\begin{bmatrix} \;\H_q&\;\H_q\\\;\H_q&-\H_q\end{bmatrix}\,.
$$
Note that the first row and column of these matrices only
consists of $+1$'s. We make use of this convenient fact in the paper.

Curiously enough this matrix is an expansion of all $\log
d$ $\pm 1$ bit patterns to $2^d$ product features. For
example, for $d=3$:

\begin{minipage}{.9\textwidth}
    $$
 d\overbrace{\left\{
     \begin{pmatrix}
    	     +&+&+\\
	     -&+&+\\
	     +&-&+\\
	     -&-&+\\
	     +&+&-\\
	     -&+&-\\
	     +&-&-\\
             -&-&-
     \end{pmatrix}
	    \right.}^{\log d} \;\stackrel{{\psi}}{\longrightarrow}\;
    \overbrace{
    \begin{pmatrix}
	+&+&+&+& +&+&+&+\\
	+&-&+&-& +&-&+&-\\
	+&+&-&-& +&+&-&-\\
	+&-&-&+& +&-&-&+\\
	+&+&+&+& -&-&-&-\\
	+&-&+&-& -&+&-&+\\
	+&+&-&-& -&-&+&+\\
	+&-&-&+& -&+&+&-
    \end{pmatrix}}^d
    $$

    \vspace{-.55cm}
    \hspace{2cm}
    \rotatebox[origin=tl]{90}{%
	\renewcommand{\arraystretch}{1.375}
    \begin{tabular}{r}
	\\[-.0cm]
	$b_1$ \\
	$b_2$ \\
	$b_3$ \\[1.45cm]
	$1$ \\
	$b_1$ \\
	$b_2$ \\
	$b_1b_2$ \\
	$b_3$\\
	$b_1b_3$ \\
	$b_2b_3$ \\
	$b_1b_2b_3$ \\
    \end{tabular}
    }
\end{minipage}

\vspace{-5cm}
Now observe that the XOR product features are hard
to learn by any time efficient algorithm from the $\log d$
bit patterns when the features are noisy. However
the above exponential expansion $\psi$
makes it possible for the EGU algorithm or the GD trained
spindly network to crack this cryptographically hard problem using essentially
$\log d$ examples (albeit in time that is exponential in $\log d$).
In the noisy case, the number of examples also only grows
with $\log d$ (\cite{wincolt20}). Note that $\log d$ is
also the Vapnik Chervonienkis dimension of the XOR problem.

Can kernel based algorithm also learn this noisy XOR
problem? First observe that the dot product between two
size $d$, $\psi$-expanded feature vectors
(i.e. rows of the Hadamard matrix) can be computed in
$O(\log d)$ time using the following reformulation of the dot product:
	$${\psi}(\bm{b})\!\cdot\!{\psi}(\mathbf{\tilde{b}})=
	    \sum_{I\subseteq 1..\log d} \;\prod_{i\in I} b_i \tilde{b}_i
	    \;=\;\prod_{i=1}^{\log d} (1\!+\!b_i\tilde{b}_i)\,.$$
This looks promising! However already in \citep{span} it
was shown that the Hadamard problem (a reformulation of the XOR problem
as seen above) cannot be learned
in a sample efficient way by any kernel method using any kernel feature map $\phi$
(including the above one): After receiving $k$
of the $d$ examples, the averaged loss on all $d$ example
is at least $1-\sfrac kd$.\footnote{%
See related discussion in the introduction of \citep{MW98}
on learning DNF formulas with the Winnow algorithm.}
The key in these lower bounds is to average over
features and examples, and also exploit the fact that
the weight space of kernel method has low rank and
the SVD spectrum of the problem matrix $\H$ is flat: $\log d\,
\underset{d,1}{\one}$.
In this paper we greatly expand this proof methodology in
Section \ref{s:sing}.
Notice that before the expansion, the SVD spectrum of
the $d \times \log d$ matrix of all bit patterns has the much
shorter flat SVD spectrum of $\log d \;\underset{\log d,1}{\one}$.

\section{Completing the proof of Theorem \ref{t:permute}}
\label{app:permutation_theorem}
    Since the target $\h$ has an equal number of $\pm$ labels, the number of
    positive unseen labels $\ELL$ is a random variable which follows
    a hyper geometric distribution, i.e.
    the number of successes in $d-k$ draws without replacement from a population of size
    $d$ that contains $\sfrac d2$ successes. Thus the mean and variance are
    \[ \mathbb{E}[\ELL] = \frac{d-k}{2}
    \quad \text{and}\quad \mathrm{Var}(\ELL)
    = \frac{(d-k)k}{4(d-1)}
    = \mathbb{E}[\ELL^2] - (\mathbb{E}[\ELL])^2 \,.
    \]

    The expected total loss on all unseen examples is thus at least
    \begin{align*}
    \frac{4}{d-k} \mathbb{E}[\ELL(d-k-\ELL)]
    &= \frac{4}{d-k} (\mathbb{E}[\ELL](d-k) - \mathbb{E}[\ELL^2]) \\
    &= \frac{4}{d-k} \left(\frac{(d-k)^2}{2} - \mathrm{Var}(\ELL) - \frac{(d-k)^2}{4} \right) \\
    &= (d-k) - \frac{k}{d-1} = d - \frac{kd}{d-1}\,.
    \end{align*}

    \vspace{-9mm}
    \Qed

\section{Proof of Theorem \ref{thm:spherically_symmetric}}
\label{app:Gaussians_proof}

We start by showing that a rotation invariant algorithm has the same average
expected (with respect to a random draw of $\X$)
loss for any $\w^{\star}$ with $\|\w^{\star}\|=1$. Indeed take two
such vectors $\w^{\star}_1$ and $\w^{\star}_2$ and note that there exists an orthogonal
transformation $\U$ such that $\w^{\star}_2 = \U \w^{\star}_1$. Due to rotation invariance
of the Gaussian distribution, the instance matrix
$\X$ has the same distribution as the instance matrix $\X \U^\top$, and thus the average
expected loss of \emph{any} algorithm on problem $(\X,\X \w^{\star}_1)$ and
$(\X \U^\top, \X \U \w^{\star}_1) = (\X \U^\top, \X \w^{\star}_2)$
is the same. Additionally, if the algorithm is rotation invariant,
its predictions (and thus the average loss) on the problems
    $(\X \U^\top,\X \w^{\star}_2)$
    and $(\X \U^\top \U,\X \w^{\star}_2) = (\X, \X \w^{\star}_2)$ are the same. Thus, we conclude that a rotation invariant algorithm
has the same average expected loss on $(\X, \X \w^{\star}_1)$ and
$(\X, \X \w^{\star}_2)$.

We now consider a problem in which $\w^{\star}$ itself is drawn uniformly from a unit sphere
in $\mathbb{R}^d$,
and show that \emph{every} algorithm (not necessarily rotation invariant) has
average expected loss at least $(1 - \sfrac kd)^2$ on all examples (where the expectation
is with respect to a random choice of $\w^{\star}$ and $\X$).
It then follows from the previous paragraph that a rotation invariant algorithm
has average expected loss at least $(1 - \sfrac kd)^2$ for any choice of $\w^{\star}$.

Thus, let $\w^{\star}$ be drawn uniformly from a unit sphere.
The covariance matrix of $\w^{\star}$ is $c\I$ from the spherical symmetry of the
distribution, and
$c$ is easily shown to be $\sfrac 1d$:
\[
cd = \tr(c \I) = \tr\left(\mathbb{E}\left[\w^{\star} \w^{\star} {}^\top\right]\right)
= \mathbb{E} \left[\tr(\w^{\star} \w^{\star} {}^\top) \right]
= \mathbb{E} \left[ \|\w^{\star}\|^2 \right] = \mathbb{E} [1] = 1\, .
\]
Let $(\X_{1:k,:}, \y_{1:k})$ be a subset of $k$ examples seen by the algorithm.
Take any unseen example $\x_i$ ($i > k$).
Decompose the loss of the algorithm on this example as
$\mathbb{E}[(y_i - \yh_i)] = \mathbb{E}[y_i^2 - 2\yh_i y_i + \yh^2]$,
where the expectation is over both $\w^{\star}$ and $\X$. The first term is easy to calculate:
\[
\mathbb{E}[y_i^2] = \mathbb{E}[\x_i^\top \w^{\star} \w^{\star}{}^\top \x_i]
= \mathbb{E}[\x_i^\top \underbrace{\mathbb{E}_{\w^{\star}}[\w^{\star} \w^{\star}{}^\top ]}_{=\I/d}\x_i]
= \frac{1}{d}\mathbb{E}[\|\x_i\|^2] = \frac{d}{d} = 1\, ,
\]
where we used the fact that $\mathbb{E}[\|\x_i\|^2] = \tr (\mathbb{E}[\x_i \x_i^\top])
= \tr (\I) = d$.

For the second term, decompose $\w^{\star} = \w_{\|} + \w_{\perp}$,
into the part within
the span of the seen examples (rows of $\X_{1:k,:}$), and the part orthogonal to
the span:
$\w_{\|} \in \text{span}\{\X_{1:k,:}\}$, $\w_{\perp}
\perp \text{span}\{\X_{1:k,:}\}$.
Note that $\w^{\star} = \w_{\|} + \w_{\perp}$
has the same distribution as $\w^{-\star} =
\w_{\|} - \w_{\perp}$. The is easily seen if we eigen decompose
$\X_{1:k,:}^\top \X_{1:k,:}
= \sum_{j=1}^r \lambda_i \v_j \v_j^\top$, where $r = \mathrm{rank}(\X_{1:k,:}) \le k$,
and form an orthogonal matrix $\U = 2\sum_{j=1}^r \v_j \v_j^\top - \I$. Since
$\v_j^\top \w_{\perp} = 0$ for any $j \le r$ and $\w_{\|}
\in \mathrm{span}(\v_1,\ldots,\v_r)$,
\[
\U \w^{\star} = \U \w_{\|} + \U \w_{\perp} = \w_{\|} -
\w_{\perp} = \w^{-\star}\, ,
\]
so the two vectors are rotations of each other, therefore having the same distribution.
Note that the prediction of the algorithm $\yh_i$ does not depend on $\w_{\perp}$,
because
the labels $\y_{\|}$ observed by the algorithm only depend on $\w_{\|}$:
$y_j = \x_j^\top \w^{\star} = \x_j^\top \w_{\|}$ for any $\x_j \in
\text{span}\{\X_{1:k,:}\}$. Thus:
\begin{align*}
\mathbb{E}[\yh_i \, y_i]
&= \mathbb{E}[\yh \, \x_i^\top \w^{\star} ]
= \frac{1}{2} \mathbb{E}[\yh \x_i^\top \, (\w^{\star} + \w^{-\star})]
= \frac{1}{2} \mathbb{E}[2\, \yh \, \x_i^\top \w_{\|}] \\
&= \mathbb{E}[\yh \, \x_{\|}^\top \w_{\|}]
= \mathbb{E}[\yh \, \x_{\|}^\top \w^{\star}]\, ,
\end{align*}
where $\x_{\|}$ is a projection of $\x_i$ onto $\text{span}\{\X_{1:k,:}\}$,
and we used the facts that $\x_i^\top \w_{\|} = (\x_{\|} + \x_{\perp})^\top
\w_{\|} = \x_{\|}^\top \w_{\|}$
and similarly
that $\x_{\|}^\top \w^{\star} = \x_{\|}^\top (\w_{\|} + \w_{\perp})
= \x_{\|}^\top \w_{\|}$.
This gives:
\[
\mathbb{E}[(y_i - \yh_i)^2] =
\mathbb{E}[y_i^2 - 2 \yh_i y_i + \yh_i^2]
= 1 + \mathbb{E}\left[\yh_i^2 -2 \, \yh_i \, \x_{\|}^\top \w^{\star}
\right]\, .
\]
The term under expectation is minimized for $\yh_i = \x_{\|}^\top \w^{\star}$:
\[
\mathbb{E}[(y_i - \yh_i)]  \ge 1 - \mathbb{E}[(\x_{\|}^\top \w^{\star})^2]\, .
\]
We now upper bound the last expectation:
\[
\mathbb{E}[(\x_{\|}^\top \w^{\star})^2]
= \mathbb{E}[\x_{\|}^\top \underbrace{\mathbb{E}_{\w^{\star}}[\w^{\star}\w^{\star}{}^\top]}_{=\I/d} \x_{\|}]
= \frac{1}{d}\, \mathbb{E} [\|\x_{\|}\|^2]
\le \frac{k}{d}\, .
\]

The last inequality follows from the fact that when conditioning on $\X_{1:k,:}$ and taking expectation with respect to $\x_i$:
\[
\mathbb{E}_{\x_i}[\|\x_{\|}\|^2]
= \sum_{j=1}^r \mathbb{E}_{\x_i} \left[ (\v_j^\top \x_i)^2 \right]
= \sum_{j=1}^r \u_j^\top \underbrace{\mathbb{E}_{\x_i}[\x_i \x_i^\top]}_{=\I} \u_j
= \sum_{i=1}^r \|\u_j\|^2 = r\, ,
\]
where $r = \mathrm{rank}(\X_{1:k,:})$ and $\{\u_1,\ldots,\u_r\}$ are the eigenvectors of $\X_{1:k,:}^\top \X_{1:k,:}$ with non-zero eigenvalues (which form a basis for $\mathrm{span}(\X_{1:k,:})$).
As $\x_i$ was chosen arbitrarily, the expected loss on every unseen example is thus
at least
$1- \frac{k}{d}$. Lower bounding the loss on the seen examples by zero, the
average expected loss is at least
\[
\frac{1}{d}\left(0 \cdot k + \left(1-\frac{k}{d}\right) \cdot (d-k) \right)
= \left(1-\frac{k}{d}\right)^2\, .
\]

\section{Proof of Theorem \ref{thm:least_squares_optimal} (optimality of least-squares)}
\label{app:LS}

The least squares algorithm predicts with
\[
\w = \X_{1:k,:}^{\dagger} \y_{1:k} =
\X_{1:k,:}^{\dagger} \X_{1:k,:} \w^{\star}\,.
\]
Note that $\X_{1:k,:}^{\dagger} \X_{1:k,:}$ is an orthogonal projection
onto span of $\X_{1:k,:}\,$.
Let $\P = \I - \X_{1:k,:}^{\dagger} \X_{1:k,:}$ be the complementary projection.
For any unseen example $\x_i$, the algorithm predicts with $\yh_t = \x_i^\top \w$
and incurs loss
\[
(y_i - \yh_i)^2 = (\x_i^{\top} \w^{\star} - \x_i^{\top} \X_{1:k,:}^{\dagger} \X_{1:k,:} \w^{\star})^2
= (\x_i^{\top} \P \w^{\star})^2
= \w^{\star}{}^\top (\P \x_i \x_i^{\top} \P) \w^{\star}\,.
\]
Taking expectation over $\x_i$:
\[
\mathbb{E}_{\x_i}[(y_i - \yh_i)^2 ]
= \w^{\star}{}^\top (\P \underbrace{\mathbb{E}_{\x_i}[\x_i \x_i^{\top}]}_{=\I} \P) \w^{\star}
= \w^{\star} \P \w^{\star}\,,
\]
where we used $\P\P = \P$ from the properties of the projection operator.
Taking expectation over $\X_{1:k,:}$ we get
$\mathbb{E}[\P] = c \I$, for some constant $c \geq 0$,
due to spherical symmetry of the distribution of $\X$. The
constant $c$ can be evaluated to $\frac{d-k}{d}$ by:
\[
cd = \tr(c \I) = \tr(\mathbb{E}[\P]) = \mathbb{E}[\tr(\P)] =
\mathbb{E}[\mathrm{rank}(\P)] = d-k\,,
\]
where we used the fact that due to spherical symmetry of the distribution, with
probability one all rows
in $\X$ are in general positions and thus any $d-k$ rows of $\X$ span a subspace of rank $d-k$.
Thus the loss on any unseen example is given
by $\mathbb{E}[(y_i - \yh_i)^2] = \frac{d-k}{d} \|\w^{\star}\|^2 = \frac{d-k}{d}$,
while the loss on every seen example is zero, as $\x_i^\top \w
= \x_i^\top \X_{1:k,:}^{\dagger} \X_{1:k,:} \w^{\star}
= \x_i^\top \w^{\star} = y_i$ for any $i \in \{1,\ldots,k\}$.
Thus, the total loss expected loss on all examples is $(d-k)(1-\sfrac{k}{d})$,
which translates to the average expected loss of
$\big(1-\frac{k}{d}\big)^2$.

\section{Proof of Theorem~\ref{thm:two-layer}}
\label{app:two-layer}
Consider a two-layer network with fully connected linear layers having weights
$\underset{d,h}{\W^{(1)}}$ and $\underset{h,1}{\w^{(2)}}$ where $d$
is the input dimension and $h$ is the number of hidden
units. Given input $\underset{k,d}{\Xtr}\!=\!\X_{1:k,:}$ and target
$\underset{k,1}{\ytr}\!=\!\y_{1:k}$, we consider the square loss
  $\half\, \Vert\underbrace{\Xtr\W^{(1)}\,\w^{(2)}-\ytr}_{\coloneqq
    \underset{k, 1}{\bm{\delta}}}\Vert^2$.
The following theorem characterized the column span of the combined weights $\W^{(1)}\w^{(2)}$ after seeing examples in $\Xtr$.
\begin{theorem}
  \label{thm:2layer-weights}
  Let $\underset{d, h}{\W_0^{(1)}}$ and
  $\underset{h, 1}{\w_0^{(2)}}$ be the initial weights of a two
    layer fully connected linear network. After seeing
    examples $\Xtr$ with GD training on any target
    $\ytr$, the weights have the form $\W^{(1)} = \W_0^{(1)} +
    \Xtr^\top (\A\,\Xtr\,\W_0^{(1)} + \a\,(\w_0^{(2)})^\top)$ and
    $\w^{(2)} = c\w_0^{(2)} + (\W_0^{(1)})^\top \Xtr^\top\b$ for some $c
    \in \RR$, $\underset{k, d}{\A}$, $\underset{k, 1}{\a}$,
    and $\underset{k, 1}{\b}$.
\end{theorem}
\begin{proof}
    The proof proceeds by induction on the updates (we used tilde to denote objects after the update):
    \begingroup
    \allowdisplaybreaks
    \begin{align*}
	\widetilde{\W}^{(1)}&=\W^{(1)} -\eta \nabla_{\W^{(1)}} L
    \\&= \W^{(1)} - \eta\Xtr^\top\del(\w^{(2)})^\top
	\\&=\W^{(1)} -\eta\Xtr^\top \del (c \w_0^{(2)} + (\W_0^{(1)})^\top\Xtr^\top\b)^\top
        \\& = \W_0^{(1)} + \Xtr^\top (\underbrace{(\A - \eta\del\b^\top)}_{\widetilde{\A}}\Xtr\W_0^{(1)} + \underbrace{(\a-c\eta\del)}_{\widetilde{\a}}\,(\w_0^{(2)})^\top)\, ,
        \\[-4mm]
        \widetilde{\w}^{(2)}&=\w^{(2)} -\eta\nabla_{\w^{(2)}} L
	\\ &= \w^{(2)} - \eta (\W^{(1)})^\top \Xtr^\top\del
	\\ &= c\w_0^{(2)} + (\W_0^{(1)})^\top \Xtr^\top\b -
	\Big(\eta (\W_0^{(1)})^\top + \eta\,
	((\W_0^{(1)})^\top\Xtr^\top\A^\top + \w_0^{(2)}\,\a^\top)\Xtr\Big)\Xtr^\top\del
        \\ &= \underbrace{(c-\eta\,\a^\top\Xtr\Xtr^\top\del)}_{\widetilde{c}}\,
	\w_0^{(2)} + (\W_0^{(1)})^\top\Xtr^\top \underbrace{(\b- \eta\,(\I+\A^\top\Xtr)\Xtr^\top\del)}_{\widetilde{\b}}\, .
    \end{align*}
    \endgroup

\vspace{-13mm}  \end{proof}

\vspace{3mm}
\begin{proof} \textbf{\!\!of Theorem}~\ref{thm:two-layer}\;
  We simply combine the observations of Theorem~\ref{thm:2layer-weights}:
  \begin{align*}
  \W^{(1)}\w^{(2)} & = \big(\W_0^{(1)} + \Xtr^\top (\A\,\Xtr\,\W_0^{(1)} + \a\,(\w_0^{(2)})^\top)\big)\big(c\w_0^{(2)} + (\W_0^{(1)})^\top \Xtr^\top\b\big)\\
  & = c\W_0^{(1)}\w_0^{(2)} + \W_0^{(1)}(\W_0^{(1)})^\top
      \Xtr^\top\b + \Xtr^\top(\A\,\Xtr\,\W_0^{(1)} +
      \a\,(\w_0^{(2)})^\top) \w^{(2)},
  \end{align*}
  which lies in the span of $[\W_0^{(1)}\w_0^{(2)},\W_0^{(1)}(\W_0^{(1)})^\top\Xtr^\top,\Xtr]$.
\end{proof}
Note that if $\W_0^{(1)}(\W_0^{(1)})^\top = \underset{d,d}{\I}$ for
$d \geq h$, then $\W^{(1)}\w^{(2)}$ lies in the column span of
$[\W_0^{(1)}\w_0^{(2)},\Xtr^\top]$ and the rank becomes at most $k
+ 1$. Alternatively for $\W_0^{(1)}=\zero$, then $\W^{(1)}\w^{(2)}$
lies in the column span of $\Xtr^\top$ which has rank at
most $k$.

In general however, the rank argument for linear
fully connected neural networks becomes weaker when the number of layers
is increased to $3, 4, \ldots,$ while experimentally, the average loss does
not improve but increasing the number of layers.

\section{Concentration of SVD spectrum of random $\pm 1$
matrices} \label{a:randpm}

Using techniques from \citep{DavSza03} and
\citep{Meckes04}, we will show that the sum of the
last $d-k$ square singular values of a random $d\times d$
$\pm 1$ matrix is concentrated around $1 - c \frac{k}{d}$, where $c>0$ is a constant
independent of $d$.

\begin{theorem}
  \label{t:ourconc}
  Let $\Mb \in \{\pm 1\}^{d \times d}$ be a random matrix and $s_{1}
  \geq s_{2} \ldots \geq s_{d}$ denote its singular values. Then, there
  is a constant $c_E\geq 1$ that does not depend on $d$ such that for
  all $t\geq c \approx 7.09$,
  \begin{align}
    \label{eq:ourconc}
    P\left[ \frac{1}{d^2}\sum_{ i = k +1}^{d} s_{i}^{2}
    \geq 1 - \frac{k}{d} \left(c_E + \frac{t}{\sqrt{d}}  \right)^{2} \right]
    \geq 1 - 4 \; \exp(-\frac{(t - c)^{2}}{4})\, .
  \end{align}
\end{theorem}
So for $t=\sqrt{d}$, the theorem says that the probability that $Q$ is
at least $1-c'\frac{k}{d}$ (for some constant $c'$ independent of $d$)
is exponentially close to 1.

The following intermediate theorem bounds the expectation and median of the largest
singular values of a random $\pm 1$ matrix. It also shows the
concentration of the matrix norm around its median. We use $\EE[X]$ and
$\MM[X]$ to denote the expectation and median of a random variable
$X$. Recall that the probability of a random variable taking a value
below the median is at most $0.5$.
\begin{theorem}[\citep{DavSza03,Meckes04}]
  \label{t:meckes}
  Let $\Mb \in \{\pm 1\}^{d \times d}$ be a random matrix, $s_{1} \geq
  s_{2} \ldots \geq s_{d}$ denote its singular values, $E := \EE[ s_{1}
  ]$, and $M := \MM[ s_{1} ]$. Then
  \begin{align}
    \label{e:E}
    \sqrt{d} \leq E \leq c_E \sqrt{d}
    \mbox{  and  }
    \sqrt{d} \leq M \leq c_M \sqrt{d}\, ,
  \end{align}
  where $c_E, c_M \geq 1$ are constants that do not depend on $d$. We
  also have that for all non-negative $t$,
  \begin{align}
    \label{eq:pmconc}
    P\left[\, | s_{1} - M | \geq t \right] \leq 4 \exp( -t^{2}/4)\, .
  \end{align}
\end{theorem}
Since $s_{1}$ is highly concentrated around its median, the following
corollary shows that it is also highly concentrated around its mean:
\begin{corollary}
  \label{cor:concmean}
  Let $\Mb$, $E$, $M$, and $s_{i}$ as above. Then for all $t\geq c =
  4^{\frac{3}{2}} \Gamma(\frac{3}{2}) \approx 7.09$,
  \begin{displaymath}
    P\left[\, |s_{1} - E| \geq t \right] \leq 4 \exp( -(t - c)^{2}/4)\,.
  \end{displaymath}
\end{corollary}
\begin{proof}
  By definition of expectation and Eq.~\eqref{eq:pmconc}
  \begin{align*}
    | E - M | = |\EE[s_{1}] - M| & \leq  \EE |s_{1} - M |\\
    & :=  \int_{0}^{\infty} P\left[ \, |s_{1} - M | \geq t \right] \, dt \\
    & \leq  4 \int_{0}^{\infty} \exp( -t^{2}/4) \, dt =  4^{\frac{3}{2}}
    \Gamma(\frac{3}{2}) = c\, .
  \end{align*}
  Using Eq.~\eqref{eq:pmconc} and the above inequality
  \begin{align*}
    P\left[\, |s_{1} - E | \geq t \right] & \leq
    P\left[\, |s_{1} - M | + | M - E | \geq t \right]
    \\
    & \leq  P\left[\, | s_{1} - M | \geq t - c \right] \\
    & \leq  4 \exp( -(t-c)^{2}/4)\, .
  \end{align*}

    \vspace{-.8cm}
\end{proof}
We are now ready to prove Theorem \ref{t:ourconc} which shows that, with
high probability, the SVD spectrum of $\Mb$ does not decay rapidly.

\vspace{2mm}
\par \noindent{\bf Proof of Theorem \ref{t:ourconc}:}
From the previous corollary, using $E, t \geq 0$, we have
\begin{align*}
  P[ s_{1}^{2} <  (E + t)^{2} ]  & \geq   P[  E - t < s_{1} < E + t ] \\
  & \geq  1 - 4 \exp(-(t-c)^{2}/4)\, ,
\end{align*}
and hence
\begin{align}
  \label{eq:prob}
  P[d^{2} - k s_{1}^{2} >  d^{2} - k (E + t)^{2} ] \geq  1 - 4
  \exp(-(t-c)^{2}/4)\,.
\end{align}
The fact that $s_{1} \geq s_{2} \ldots \geq s_{d}$ and $\sum_{i}
s_{i}^{2} = ||\Mb||^{2}_{F} := \sum_{i,j} |M_{i,j}|^{2} = d^{2}$,
jointly imply that $\sum_{i > k} s_{i}^2 \geq d^{2} - k s_{1}^2$.
Using this and Eq.~\eqref{eq:prob}, we obtain
\begin{align*}
    P\left[ \frac{1}{d^2} \sum_{i = k + 1}^{d} s_{i}^{2} \geq 1 -
    \frac{k}{d^2} (E + t)^{2} \right] \geq 1 - 4 \exp(-(t-c)^{2}/4)\,.
\end{align*}
Since $E\leq c_E \sqrt{d}$ (see Eq.\eqref{e:E}), the LHS of the above
inequality is upper bounded by $ P\left[ \frac{1}{d^2}\sum_{ i = k
+1}^{d} s_{i}^{2} \geq 1 - \frac{k}{d} \left(c_E +
\frac{t}{\sqrt{d}} \right)^{2} \right] $ and the theorem follows.
\hfill \BlackBox

\end{document}